\documentclass{article}

\PassOptionsToPackage{numbers, compress}{natbib}
\usepackage[preprint]{neurips_2026}

\usepackage[utf8]{inputenc} 
\usepackage[T1]{fontenc}    
\usepackage{hyperref}       
\usepackage{url}            
\usepackage{booktabs}       
\usepackage{amsfonts}       
\usepackage{nicefrac}       
\usepackage{microtype}      
\usepackage{xcolor}         
\usepackage{graphicx}
\usepackage{subcaption}
\usepackage{makecell}
\usepackage{bm}
\usepackage{booktabs,multirow,xcolor}
\usepackage{amsmath}
\usepackage{wrapfig}
\usepackage{amsthm}
\usepackage{bbm}
\usepackage{hyperref}

\newtheorem{theorem}{Theorem}

\newtheorem{remark}{Remark}

\title{CHE-TKG: Collaborative Historical Evidence \\ and Evolutionary Dynamics Learning \\ for Temporal Knowledge Graph Reasoning}

\author{%
  Shuai-long Lei, Xiaobin Zhu\thanks{Corresponding Author}, Jiarui Liang, Guoxi Sun, Zhiyu Fang, Xu-Cheng Yin\\
  University of Science and Technology Beijing \\
  \texttt{shuailong0lei@gmail.com,zhuxiaobin@ustb.edu.cn,m202510664@xs.ustb.edu.cn}\\
  \texttt{sgxxyyds@foxmail.com,mr.fangzy@foxmail.com,xuchengyin@ustb.edu.cn} \\
}

\begin{document}

\maketitle

\vspace{-10pt}
\begin{abstract}
Temporal knowledge graph (TKG) reasoning aims to predict future events from historical facts.
A key challenge lies in jointly capturing two sources of predictive information in TKGs: historical evidence and evolutionary dynamics.
However, existing methods typically focus on only one of these sources, which limits the ability to fully exploit the complementary predictive signals in TKGs.
To address this, we propose CHE-TKG, a novel collaborative dual-view learning framework for TKG reasoning. 
CHE-TKG explicitly separates and jointly models historical evidence and evolutionary dynamics, aiming to learn and exploit their complementary predictive signals.
Specifically, CHE-TKG constructs a historical evidence graph to capture long-term structural regularities and stable relational constraints, alongside an evolutionary dynamics graph to model temporal transitions and recent changes, with dedicated encoders for each view.
We further employ relation decomposition and a contrastive alignment objective to better capture the predictive signals across the two views.
Extensive experiments demonstrate that CHE-TKG achieves state-of-the-art performance on multiple benchmarks.

\end{abstract}

\section{Introduction}
\vspace{-2pt}
Temporal knowledge graphs (TKGs) model real-world facts with temporal information, where each fact is represented as a quadruple $(subject, relation, object, timestamp)$ and organized as a sequence of time-ordered snapshots.
This temporal modeling supports a wide range of downstream applications, such as recommendation \cite{tang2024editkg,wang2023mixed}, question answering \cite{atif2023beamqa,liu2022joint}, and information retrieval \cite{xu2024retrieval,ding2024enhancing}.
Due to the inherent incompleteness of TKGs \cite{zhang2025historically,du2025hawkes,chen2024local,liang2023learn}, temporal knowledge graph reasoning (TKGR) aims to infer missing facts at specific timestamps. 

TKGs often exhibit two distinct sources of predictive information \cite{liu2025terdy, lacroixtensor, chenbeyond}, namely historical evidence (E) and evolutionary dynamics (D), which may provide complementary predictive signals for TKGR.
Complementary predictive signals refer to different yet useful predictive cues captured from these two sources.
Recent TKGR methods attempt to leverage structural inductive biases to capture these two aspects, as shown in Fig.~\ref{fig:problem}(b).
One line of work \cite{li2022tirgn, chen2024local, zhang2024temporal, zhang2025historically} focuses on historical evidence, where historically observed relevant events are aggregated while discarding temporal information. Another line of research \cite{du2025hawkes, chen2024htccn} emphasizes evolutionary dynamics, typically leveraging various mechanisms to model temporally ordered events, thereby capturing time-evolving dependencies.
However, these methods typically focus on only one of these sources, which limits the ability to fully exploit the complementary predictive signals in TKGs.

{
  \setlength{\belowcaptionskip}{-10pt}
\begin{figure}[t]
  \centering
  \includegraphics[width=\columnwidth]{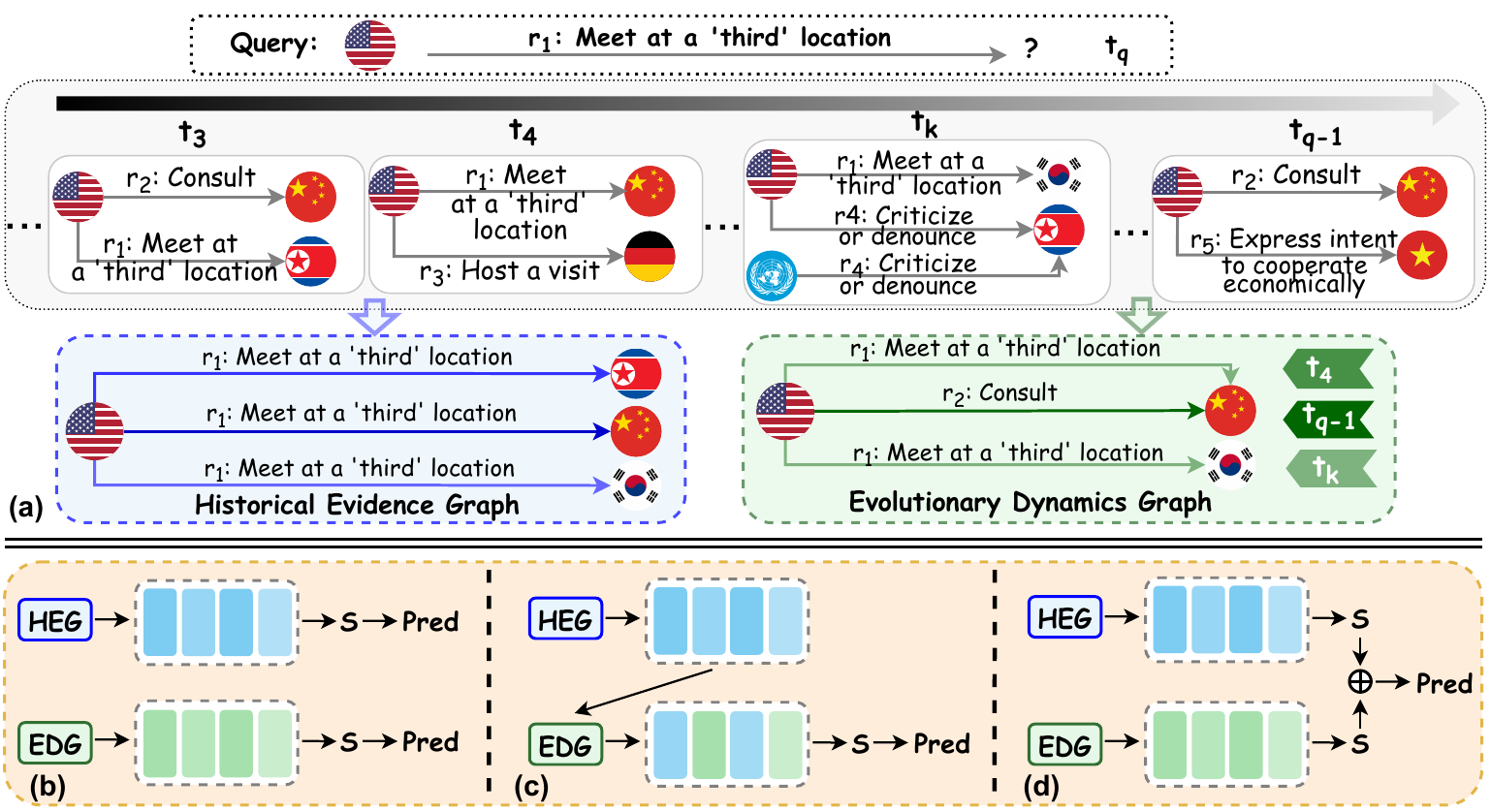}
  \caption{An illustrative example of the two views, including graph construction on the ICEWS dataset and comparisons of different reasoning paradigms.
  (a) Construction of a historical evidence graph and an evolutionary dynamics graph from the observed TKG given a query.
  (b) Single-view reasoning paradigm. After obtaining entity and relation embeddings, prediction is performed by computing scores.
  (c) Naive joint modeling of both views.
  (d) Our CHE-TKG with explicit separation and joint learning of two views.
  }
  \label{fig:problem}
\end{figure}
}

To better understand these two aspects, we instantiate them as two graph structures in a TKG: a historical evidence graph and an evolutionary dynamics graph.
Fig.~\ref{fig:problem}(a) illustrates this construction using an example from the ICEWS dataset.
Given the query $q=$ (\textit{North America}, \textit{Meet at a `third' location}, ?, $t_q$), historical facts of the form (\textit{North America}, \textit{Meet at a `third' location}, $e_o$, t) with $t < t_q$ reflect accumulated historical evidence.
Removing timestamps and aggregating these facts corresponds to a historical evidence graph (blue dashed region), which captures long-term structural regularities and stable relational constraints.
In parallel, logically related recent events, derived from temporal logical rules \cite{liu2022tlogic}, capture the evolutionary dynamics, which can be organized as an evolutionary dynamics graph (green dashed region) to capture query-relevant temporal transitions. 
Although naively combining these two structures enables joint learning, such integration still poses a significant challenge, as shown in Fig.~\ref{fig:problem}(c). In particular, simple joint modeling may suffer from shortcut learning, where the model tends to favor easier-to-learn patterns while failing to preserve informative signals \cite{pezeshki2021gradient, scimecashortcut}. 
As a result, complementary predictive signals may be under-utilized or collapsed in the learned representations, preventing them from being fully exploited.

To address the aforementioned challenges, we propose CHE-TKG, a novel collaborative dual-view learning framework for TKGR.
CHE-TKG explicitly separates and jointly models historical evidence and evolutionary dynamics through two dedicated structural views, thereby better learning and exploiting complementary predictive signals, as shown in Fig.~\ref{fig:problem}(d).
Specifically, we construct a historical evidence graph and an evolutionary dynamics graph, and employ two dedicated encoders to learn view-specific representations.
The historical evidence graph captures long-term structural regularities and stable relational constraints, while the evolutionary dynamics graph models temporal transitions and recent changes.
We also incorporate relation decomposition to learn view-specific relation semantics, and introduce contrastive alignment to encourage the two views to capture task-relevant signals.
Our main contributions are summarized as follows:
\begin{itemize}
  \item We propose a novel collaborative dual-view learning framework for TKGR (dubbed CHE-TKG) to better learn and exploit complementary predictive signals from historical evidence and evolutionary dynamics.
  
  \item We construct historical evidence and evolutionary dynamics graphs with dedicated encoders, and introduce relation decomposition and contrastive alignment to learn view-specific semantics and task-relevant cross-view signals.

  \item We theoretically show that the collaborative dual-view learning framework can improve ranking reliability by reducing pair-level ranking errors under mild conditions.

  \item Extensive experiments on multiple benchmark datasets verify that CHE-TKG consistently achieves state-of-the-art performance in extrapolation settings.
\end{itemize}
  
\section{Related Work}

Since TKGs are inherently graph-structured, many extrapolation reasoning methods adopt graph neural networks (GNNs) to learn entity and relation embeddings.
RE-NET \cite{jin2020recurrent} and RE-GCN \cite{li2021temporal} pioneered the paradigm by modeling spatial structures with GNNs and capturing local temporal context across snapshots using recurrent neural networks (RNNs). 
Building upon this paradigm, subsequent works explore two main directions. Some approaches focus on relation-centric modeling to improve entity representations \cite{chen2021dacha,zhang2023learning,liu2023retia,liang2023learn}, while others \cite{li2022hismatch,li2022complex,sun2022graph} strengthen temporal modeling across recent snapshots through various designs. 
However, these methods do not explicitly instantiate historical evidence or evolutionary dynamics as dedicated structural views.

Recent approaches attempt to explicitly model either historical evidence or evolutionary dynamics by introducing extra graph structures. Specifically, HGLS \cite{zhang2023learning_2} builds a global graph over a large temporal window, while DyMemR \cite{zhang2024temporal} maintains a memory pool to selectively store global historical information. TiRGN \cite{li2022tirgn} introduces a global historical replay matrix, whereas LogCL \cite{chen2024local} and HisRES \cite{zhang2025historically} model global two-hop and one-hop historical dependencies, respectively.
These approaches primarily leverage query-related, time-agnostic global structures to model historical evidence, improving the utilization of long-term contextual information. Meanwhile, another line of work emphasizes temporal evolution through dynamic structures. For instance, HERLN \cite{du2025hawkes} combines Hawkes processes with structural properties such as community and long-tail distributions, and HTCCN \cite{chen2024htccn} adopts causal convolutions and Hawkes processes to capture temporal dependencies, temporal decay, and event intensities.
However, existing approaches do not unify the modeling of historical evidence and evolutionary dynamics, leaving a gap in jointly modeling and leveraging the complementary predictive signals from the two sources.
Further details on other TKGR methods can be found in Appendix~\ref{appe:related}.

\section{Methodology}
\subsection{Preliminaries}
\textbf{Task Definition.}
A temporal knowledge graph (TKG) is defined as $(\mathcal{E}, \mathcal{R}, \mathcal{T}, \mathcal{F})$, where $\mathcal{E}$, $\mathcal{R}$, $\mathcal{T}$, and $\mathcal{F}$ denote the sets of entities, relations, timestamps, and temporal facts, respectively. 
Each fact $f \in \mathcal{F}$ is a quadruple $(e_s, r, e_o, t)$ with $e_s, e_o \in \mathcal{E}$, $r \in \mathcal{R}$, and $t \in \mathcal{T}$. 
Facts are organized into time-ordered snapshots $\mathcal{G}=\{\mathcal{G}_1, \mathcal{G}_2, \dots\}$, where $\mathcal{G}_t=\{(e_s, r, e_o, t)\mid (e_s, r, e_o, t)\in\mathcal{F}\}$.
Extrapolation reasoning aims to infer facts at unseen future timestamps and is formulated as a link prediction task that answers queries $q=(e_s, r, ?, t_q)$. During training, only historical facts with $t_i < t_q$ are observable, while facts at timestamp $t_q$ and beyond remain unobserved.
We denote by $\mathcal{Q}_{t_q}$ the set of such queries at timestamp $t_q$.
Following prior work, inverse relations are introduced by adding edges $(e_o, r^{-1}, e_s, t)$, allowing subject prediction queries to be transformed into tail prediction queries. 
For simplicity, we describe the task in terms of tail entity prediction in the subsequent analysis. 
A summary of important notations is provided in Appendix~\ref{abla:nota}.

\textbf{Temporal Logical Rule.}
A temporal logical rule $\mathcal{TR}$ is defined as
$
(A, r_h, B, t_2) \leftarrow (A, r_b, B, t_1), \quad t_2 > t_1,
$
indicating that the occurrence of relation $r_b$ between entities $A$ and $B$ at time $t_1$ supports the occurrence of relation $r_h$ between the same entities at a later time $t_2$. The confidence of a rule is measured by how often the rule
body is followed by the rule head.
\subsection{Framework}
{
  \setlength{\belowcaptionskip}{-10pt}
\begin{figure*}[ht!]
  \centering
  \includegraphics[width=\textwidth]{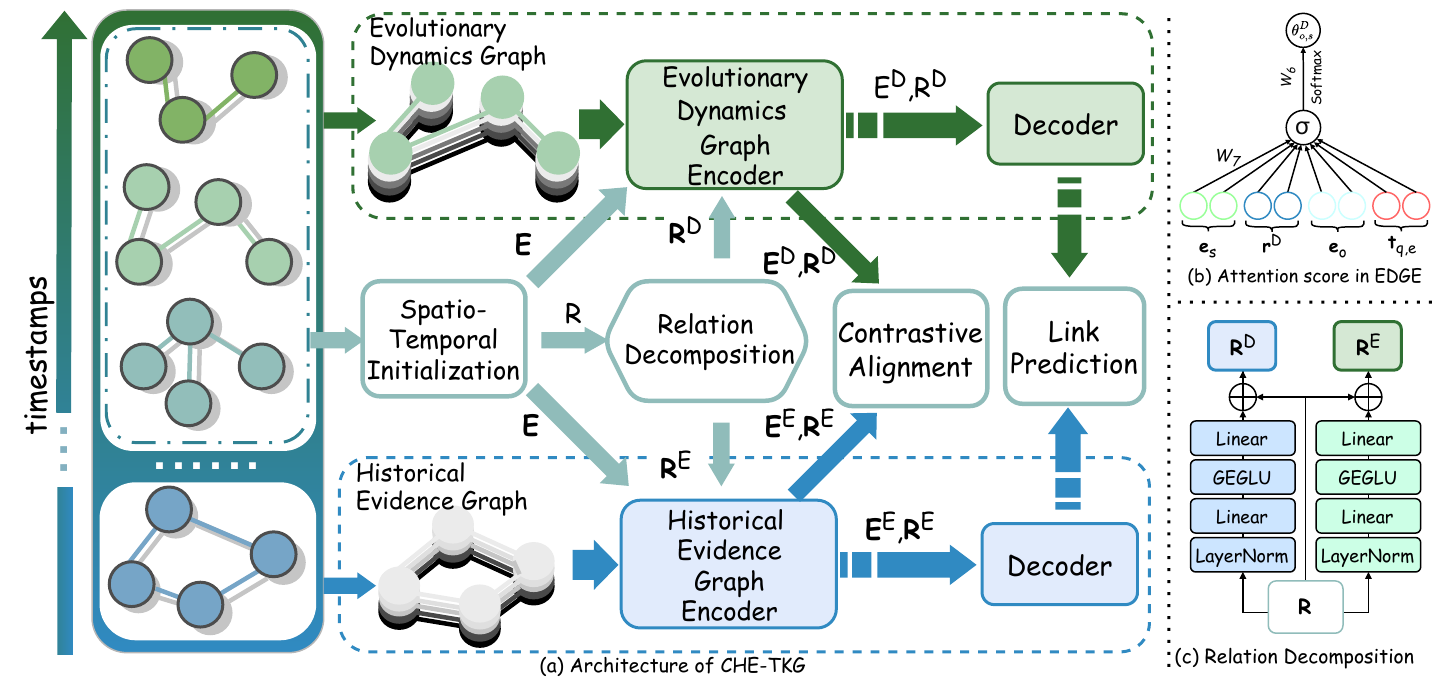}
  \vspace{-16pt}
  \caption{Overall architecture of CHE-TKG (a). The framework consists of
  spatio-temporal initialization, relation decomposition (c),
  a historical evidence graph encoder, and an evolutionary dynamics graph encoder
  with time-aware attention scores (b), followed by contrastive
  alignment.
  }
  \label{fig:pipeline}
\end{figure*}
}

We propose a collaborative dual-view learning framework CHE-TKG for TKGR, as illustrated in Fig.~\ref{fig:pipeline}. 
CHE-TKG constructs a historical evidence graph and an evolutionary dynamics graph from observed facts to model two complementary sources of predictive information. 
The historical evidence graph captures long-term structural regularities and stable relational patterns, while the evolutionary dynamics graph captures recent temporal transitions and time-sensitive changes.
Facts from recent snapshots are incorporated to initialize spatio-temporal entity and relation representations, providing general contextual information for subsequent collaborative reasoning. 
After initialization, relation representations are decomposed into view-specific components, enabling a relation to carry different semantics in the historical and evolutionary views. 
Then, the Historical Evidence Graph Encoder (HEGE) and Evolutionary Dynamics Graph Encoder (EDGE) encode the decomposed relations and entity embeddings to learn view-specific structural signals from the two graphs. 
A contrastive alignment mechanism further aligns query representations across views to promote task-relevant complementarity. 
Finally, the dual-view representations are combined in the decoder for entity prediction, allowing CHE-TKG to exploit complementary predictive signals from both views.

\subsection{Collaborative Graph Construction}

\textbf{Historical Evidence Graph.} We construct the historical evidence graph by aggregating historical events while discarding explicit temporal information. 
Specifically, given a query $q=(e_s, r, ?, t_q)$, we retrieve all historical facts
$f^{E}_{t_q}=(e_s, r, e_o, t)$ with $t<t_q$ and incorporate them into the corresponding historical evidence subgraph $G^{E}_{t_q}\in\mathcal{G}^{E}$. 
As $t_q$ increases, the historical evidence graph and the evolutionary dynamics graph are constructed from an expanded set of historical facts observed prior to $t_q$.
For each timestamp $t_q$, all facts associated with its queries are aggregated into the same historical evidence and evolutionary dynamics subgraphs, denoted as $G^{E}_{t_q}$ and $G^{D}_{t_q}$.

\textbf{Evolutionary Dynamics Graph.}
We construct the evolutionary dynamics subgraph $G^{D}_{t_q}$ using a temporal logical rule--based retrieval strategy.
Specifically, we first mine temporal logical rules using TLogic \cite{liu2022tlogic} and sort them by confidence in descending order.
Given a query $q=(e_s,r,?,t_q)$, the query relation $r$ is treated as the head relation $r_h$, and rules with head relation $r_h$ are selected accordingly.
For each selected rule with body relation $r_b$, we retrieve historical facts of the form $f^{D}_{t_q}=(e_s, r_b, e_b, t)$ with $t<t_q$ in reverse chronological order, where $e_b$ denotes an entity connected to $e_s$ via the rule body relation.
These facts are then incorporated into $G^{D}_{t_q}$ until a predefined upper bound $N$ is reached.
For each edge, we explicitly encode the relative time interval $\Delta t=t_q-t$ as edge information, allowing the model to capture fine-grained temporal evolution patterns. 
Additional analysis of rule-based retrieval is provided in Appendix~\ref{abla:rule}.

\subsection{Initialization}

The initialization stage provides spatio-temporal representations by leveraging temporal and structural information from recent snapshots. 
Given the snapshot at time $t-1$, we employ a graph convolutional network (GCN) to aggregate multi-relational neighborhood information within each snapshot:
\begin{equation}
\mathbf{e}_{o,t-1}^{\,l+1}
=
\sigma(
\sum_{(e_s,r,e_o)\in\mathcal{G}_{t-1}}
\frac{1}{c_o}
\, W_1^{\,l}
\kappa(\mathbf{e}_{s,t-1}^{\,l}, \mathbf{r}_{t-1}^{\,l})
+
W_2^{\,l} \mathbf{e}_{o,t-1}^{\,l}
),
\end{equation}
where $\mathbf{e}_{s,t-1}^{\,l}$ and $\mathbf{e}_{o,t-1}^{\,l}$ are entity embeddings at time $t\!-\!1$ in the $l$-th layer, $W_1^{\,l}$ and $W_2^{\,l}$ are learnable parameters, $c_o$ denotes the in-degree normalization factor, and $\sigma(\cdot)$ is the RReLU activation function. 
Following \cite{zhang2025historically}, relation embeddings are updated jointly in the GCN, and $\kappa(\cdot)$ denotes a one-dimensional convolution operator as in \cite{li2022tirgn}.
After modeling snapshot-level spatial structures with GCNs, we employ Gated Recurrent Units (GRUs) to capture temporal dependencies across snapshots. Specifically, the entity and relation representations are updated as:
\begin{equation}
\mathbf{E}_t = \mathrm{GRU}_E(\mathbf{E}_{t-1}, \mathbf{E}_{t-1}^{\mathrm{GCN}}), \quad
\mathbf{R}_t = \mathrm{GRU}_R(\mathbf{R}_{t-1}, \mathbf{R}_t').
\end{equation}
where $\mathbf{E}_t$ and $\mathbf{E}_{t-1}$ denote the entity embedding matrices at time steps $t$ and $t-1$, $\mathbf{E}_{t-1}^{\mathrm{GCN}}$ is the GCN output at time $t-1$, and $\mathbf{R}_t'$ denotes the intermediate relation representation matrix. For each relation $r$, its intermediate representation is constructed as
$
\mathbf{r}_t' = [\, \mathrm{pool}(\mathbf{E}_{r,t}) \,\Vert\, \mathbf{r} \,],
$
where $\mathbf{E}_{r,t}$ denotes the embeddings of entities associated with relation $r$ at time $t$, and $\mathrm{pool}(\cdot)$ denotes mean pooling.

\subsection{Relation Decomposition}
To learn heterogeneous relational semantics and preserve view-specific predictive signals across the two graph views, we introduce a relation decomposition mechanism that learns view-specific relation representations. Formally, given a base relation embedding $r$, we derive the view-specific relation embedding for each view $v \in \{E,D\}$ via multiplicative modulation:
\begin{equation}
  \hat{r}^{v}
  =
  \bigl( \mathbf{1} + g^{v}(r) \bigr) \odot r,
  \qquad
  g^{v}(r)
  =
  W^{v}_3
  \,
  \mathrm{Drop}\!\left(
  \mathrm{GEGLU}(\mathrm{LN}(r))
  \right),
  \quad v \in \{E,D\}.
\label{eq:relation_decomposition}
\end{equation}
where $\hat{r}^{E}$ and $\hat{r}^{D}$ are used in the historical evidence graph and the evolutionary dynamics graph, respectively, $\mathbf{1}$ denotes an all-ones vector, and $\odot$ denotes element-wise multiplication. Here, $W^{v}_3$ is a view-specific learnable projection matrix, $\mathrm{LN}(\cdot)$ denotes layer normalization, $\mathrm{GEGLU}(\cdot)$ \cite{shazeer2020glu} applies a gated linear unit with GELU activation, and $\mathrm{Drop}(\cdot)$ denotes dropout.

\subsection{Graph Encoder}
To capture event correlations in both graphs, we adopt a Graph Attention Network (GAT) architecture with two view-specific encoders: the Historical Evidence Graph Encoder (HEGE) and the Evolutionary Dynamics Graph Encoder (EDGE). 
Entity and relation representations are initialized with spatio-temporal and decomposed embeddings, respectively.
To model temporal relevance, we encode the relative time interval between the query
$q=(e_s,r,?,t_q)$ and a historical fact $(e_s,r,e_o,t_e)$ as:
$
t_{q,e}=\Phi(t_q-t_e),
$
where $\Phi(\cdot)$ follows TGAT \cite{xu2020inductive}:
$
\Phi(t)=\sqrt{\tfrac{1}{d}}\,[\cos(w_1 t+p_1),\ldots,\cos(w_d t+p_d)] .
$
The attention coefficients in the EDGE and HEGE are then defined as:
\begin{equation}
  \theta^{D,l}_{o,s}
  =
  \frac{
  \exp\!\big(
  W^{l}_{4}\,
  \sigma\!\big(
  W^{l}_{5}\,
  [\,\mathbf{e}^{l}_{s} \Vert \hat{\mathbf{r}}^{D} \Vert \mathbf{e}^{l}_{o} \Vert t_{q,e}\,]
  \big)
  \big)
  }{
  \sum_{(e_s',r') \in \mathcal{N}^{D}(o)}
  \exp\!\big(
  W^{l}_{4}\,
  \sigma\!\big(
  W^{l}_{5}\,
  [\,\mathbf{e}^{\prime l}_{s} \Vert \hat{\mathbf{r}}^{\prime D} \Vert \mathbf{e}^{l}_{o} \Vert t_{q,e'}\,]
  \big)
  \big)
  } .
\end{equation}
\begin{equation}
  \theta^{E,l}_{o,s}
  =
  \frac{
  \exp\!\big(
  W^{l}_{6}\,
  \sigma\!\big(
  W^{l}_{7}\,
  [\,\mathbf{e}^{l}_{s} \Vert \hat{\mathbf{r}}^{E} \Vert \mathbf{e}^{l}_{o}\,]
  \big)
  \big)
  }{
  \sum_{(e_s^{\prime},r) \in \mathcal{N}^{E}(o)}
  \exp\!\big(
  W^{l}_{6}\,
  \sigma\!\big(
  W^{l}_{7}\,
  [\,\mathbf{e}^{\prime l}_{s} \Vert \hat{\mathbf{r}}^{E} \Vert \mathbf{e}^{l}_{o}\,]
  \big)
  \big)
  } .
\end{equation}
For both graph views $v\in\{E,D\}$, entity representations are updated by:
\begin{equation}
  \mathbf{e}^{\,v,l+1}_{o}
  =
  \sigma(
  \sum_{(e_s,r,e_o)\in \mathcal{G}^{v}_{t_q}}
  \theta^{\,v,l}_{o,s}\,
  W^{\,v,l}_{8}\,
  \psi\!\left(\mathbf{e}^{\,v,l}_{s} + \hat{\mathbf{r}}^{\,v,l}\right)
  \;+\;
  W^{\,v,l}_{9}\,
  \mathbf{e}^{\,v,l}_{o}
  ),
\end{equation}
where $\mathbf{e}^{v,l}_{s}$ and $\mathbf{e}^{v,l}_{o}$ denote the source and target entities' embeddings at layer $l$, $\hat{\mathbf{r}}^{v,l}$ denotes the decomposed relation embedding, $\mathcal{N}^{v}(o)$ denotes the incoming neighbors of $e_o$ in graph view $v$, $W^{(\cdot)}$ are learnable parameters, $\psi(\cdot)$ denotes a 1-D convolution operator, and $\sigma(\cdot)$ is an activation function.

\subsection{Contrastive Alignment}

After encoding the two graph views with EDGE and HEGE, we obtain entity and relation representations from both graphs.
Given a query $q=(e_s,r,?,t_q)\in\mathcal{Q}_{t_q}$, we construct view-specific query representations and align them using an InfoNCE-based contrastive objective:
\begin{equation}
\begin{aligned}
\mathbf{z}^{v}_q
&=
\mathrm{MLP}^{v}\big([\mathbf{e}^{v}_{s} \Vert \mathbf{r}^{v}]\big),
\quad v\in\{D,E\},\\
\mathcal{L}_{D \rightarrow E}
&=
- \mathbb{E}_{q \in \mathcal{Q}_{t_q}}
\log
\frac{
\exp\!\left( \mathbf{z}^{D}_q \cdot \mathbf{z}^{E}_q / \gamma \right)
}{
\sum_{q' \in \mathcal{Q}_{t_q}}
\exp\!\left( \mathbf{z}^{D}_q \cdot \mathbf{z}^{E}_{q'} / \gamma \right)
}.
\end{aligned}
\label{eq:contrastive_alignment}
\end{equation}
where $\mathrm{MLP}^{v}$ is implemented with an LN-GEGLU-Drop-Linear block with independent parameters, and $\gamma$ is a temperature parameter.
The final contrastive objective is defined symmetrically as $\mathcal{L}_{\mathrm{CoA}}=\mathcal{L}_{D \rightarrow E}+\mathcal{L}_{E \rightarrow D}$, where $\mathcal{L}_{E \rightarrow D}$ is defined analogously.
This contrastive alignment encourages the two views to capture task-relevant cross-view signals rather than task-irrelevant noise.

\subsection{Prediction and Optimization}

For link prediction, CHE-TKG adopts ConvTransE following \cite{li2021temporal,chen2024local,zhang2025historically} as the scoring function to compute the compatibility score between a query $(e_s,r,?,t)$ and a candidate entity $e_o$.
Given representations learned from the two graph views, the final prediction score is obtained by summing the view-specific scores:
\begin{equation}
f(e_o \mid e_s, r, t)
=
\sum_{v\in\{E,D\}}
f^{v}(e_o \mid e_s, r, t),
\qquad
f^{v}(e_o \mid e_s, r, t)
=
\operatorname{ConvTransE}\!\left(\mathbf{e}^{v}_{s,t}, \mathbf{r}^{v}_t\right)
\cdot \mathbf{e}^{v}_{o,t}.
\label{eq:prediction_score}
\end{equation}
where $\mathbf{e}^{v}_{s,t}$, $\mathbf{r}^{v}_t$, and $\mathbf{e}^{v}_{o,t}$ denote the subject entity, relation, and candidate object representations in view $v$, respectively.
To enhance relation representations, we further introduce relation prediction as an auxiliary task.
Let $\mathcal{Q}^{e}_{t}$ and $\mathcal{Q}^{r}_{t}$ denote the sets of entity and relation prediction queries at timestamp $t$, respectively.
The task objective is defined as:
\begin{equation}
\mathcal{L}_{\mathrm{TKG}}
=
\alpha
\sum_{(e_s,r,e_o,t)\in \mathcal{Q}_t^{e}}
\mathcal{L}_{\mathrm{CE}}\!\left(e_o \mid e_s, r, t\right)
+
(1-\alpha)
\sum_{(e_s,r,e_o,t)\in \mathcal{Q}_t^{r}}
\mathcal{L}_{\mathrm{CE}}\!\left(r \mid e_s, e_o, t\right).
\label{eq:rel_pred}
\end{equation}
where $\mathcal{L}_{\mathrm{CE}}(\cdot)$ denotes the cross-entropy loss, $\alpha\in[0,1]$ balances entity and relation prediction tasks.
Finally, the overall training objective combines the task loss with the contrastive alignment loss:
\begin{equation}
\mathcal{L}
=
\mathcal{L}_{\mathrm{TKG}}
+
\mu\,\mathcal{L}_{\mathrm{CoA}},
\end{equation}
where $\mu\ge0$ controls the contribution of contrastive alignment.

\subsection{Theoretical Analysis}

Temporal knowledge graph reasoning under the extrapolation setting can be formulated as a ranking problem.
Since CHE-TKG produces two view-specific scores from historical evidence and evolutionary dynamics, we analyze the corresponding pair-level score margins to understand when such dual-view signals improve ranking reliability.
Given a query $q=(s,r,?,t)$, the model aims to assign a higher score to the ground-truth entity $o^+$ than to a negative candidate $o^-$.
For each query--negative pair $(q,o^-)$, ranking correctness is characterized by the margin:
\begin{equation}
\gamma(q,o^-)=S(q,o^+)-S(q,o^-),
\end{equation}
where $\gamma(q,o^-)>0$ indicates a correct ranking and $\gamma(q,o^-)\le0$ indicates a pairwise ranking error.
Thus, increasing the ranking margin is crucial for improving ranking metrics.

CHE-TKG performs prediction from two views: the evolutionary dynamics view and the historical evidence view, which produce scores $S_D(q,o)$ and $S_E(q,o)$, respectively.
The final score is obtained via score-level fusion, and the corresponding view-specific and fused margins are defined as:
\begin{equation}
\begin{gathered}
S_{D+E}(q,o)=S_D(q,o)+S_E(q,o), \\
\gamma_v(q,o^-)=S_v(q,o^+)-S_v(q,o^-), \quad v\in\{D,E\}, \\
\gamma_{D+E}(q,o^-)=S_{D+E}(q,o^+)-S_{D+E}(q,o^-).
\end{gathered}
\end{equation}
By score additivity, the fused margin satisfies:
\begin{equation}
\gamma_{D+E}(q,o^-)=\gamma_D(q,o^-)+\gamma_E(q,o^-).
\end{equation}
This shows that score-level fusion aggregates ranking signals from both views.
Intuitively, if the two views provide positive yet non-identical ranking signals, the fused margin is expected to increase to reduce pairwise ranking errors.
We formalize this intuition through the following result.
For analytical clarity, we suppose that the two margins follow identical Gaussian marginals,
$
\gamma_D, \gamma_E \sim \mathcal{N}(\mu,\sigma^2).
$
\newpage
\begin{theorem}[Ranking Error Reduction]\label{theo:error}
Let $\gamma_D$ and $\gamma_E$ denote the pairwise ranking margins induced by the evolutionary dynamics view and the historical evidence view, respectively. 
Under the Gaussian margin model above, if the two view-specific margins satisfy
$
\mathbb{E}[\gamma_D] > 0, 
\mathbb{E}[\gamma_E] > 0,
$
and
$
\operatorname{Corr}(\gamma_D,\gamma_E)=\rho < 1,
$
then the fused margin $\gamma_{D+E}=\gamma_D+\gamma_E$ satisfies
\begin{equation}
P(\gamma_{D+E}\le 0) < P(\gamma_D\le 0) = P(\gamma_E\le 0).
\end{equation}
\end{theorem}

The first condition in Theorem~\ref{theo:error}, namely
$\mathbb{E}[\gamma_D]>0$ and $\mathbb{E}[\gamma_E]>0$, characterizes the \textbf{View Effectiveness} property: each view provides positive ranking signals on average. 
The second condition, $\operatorname{Corr}(\gamma_D,\gamma_E)=\rho<1$, characterizes the \textbf{View Non-redundancy} property: the ranking signals from the two views are not perfectly redundant. 
We refer to signals satisfying both properties as \textbf{complementary predictive signals}, i.e., effective and non-redundant pairwise ranking margins.
CHE-TKG is designed to learn such complementary predictive signals by explicitly separating historical evidence and evolutionary dynamics into two view-specific structural views.
Relation decomposition preserves view-specific relation semantics, while contrastive alignment maintains prediction-aligned semantic consistency across views, jointly encouraging the two views to be both effective and non-redundant.
Complementary predictive signals are empirically validated in Section~\ref{exp:signal}, and the proof of Theorem~\ref{theo:error} is provided in Appendix~\ref{proof}.

\begin{table*}[]
  \centering
  \setlength{\tabcolsep}{3pt} 
  \renewcommand{\arraystretch}{1.0}  
  \caption{Performance of TKG entity extrapolation on ICEWS14s, ICEWS18, ICEWS05-15 and GDELT. The best results are marked in \textbf{bold}, and the second-best results are \underline{underlined}.}
  \label{tab:main-results}

  \resizebox{1.0\textwidth}{!}{
  \begin{tabular}{l *{3}{c} *{3}{c} *{3}{c} *{3}{c}}
  \toprule
  \multirow{2}{*}{\textbf{Model}}
  & \multicolumn{3}{c}{\textbf{ICEWS14s}}
  & \multicolumn{3}{c}{\textbf{ICEWS18}}
  & \multicolumn{3}{c}{\textbf{ICEWS05-15}}
  & \multicolumn{3}{c}{\textbf{GDELT}} \\
  \cmidrule(lr){2-4}\cmidrule(lr){5-7}\cmidrule(lr){8-10}\cmidrule(lr){11-13}

  & \textbf{MRR} & \textbf{Hits@1} & \textbf{Hits@10}
  & \textbf{MRR} & \textbf{Hits@1} & \textbf{Hits@10}
  & \textbf{MRR} & \textbf{Hits@1} & \textbf{Hits@10}
  & \textbf{MRR} & \textbf{Hits@1} & \textbf{Hits@10} \\
  \midrule

  RE-NET & 36.93 & 26.83 & 54.78 & 29.78 & 19.73 & 48.46 & 43.67 & 33.55 & 62.72 & 19.55 & 12.38 & 34.00 \\
  RE-GCN & 42.39 & 31.96 & 62.29 & 32.62 & 22.39 & 52.68 & 48.03 & 37.33 & 68.51 & 19.69 & 12.46 & 33.81 \\
  RETIA  & 42.76 & 32.28 & 62.75 & 32.43 & 22.23 & 52.94 & 47.26 & 36.64 & 67.76 & 20.12 & 12.76 & 34.49 \\
  RPC    & --    & --    & --    & 34.91 & 24.34 & 55.89 & 51.14 & 39.47 & 71.75 & 22.41 & 14.42 & 38.33 \\
  TANGO  & 36.80 & 27.43 & 54.93 & 28.68 & 19.35 & 47.04 & 42.86 & 32.72 & 62.34 & 19.53 & 12.43 & 33.19 \\
  HisMatch & 46.42 & 35.91 & 66.84 & 33.99 & 23.91 & 53.94 & 52.85 & 42.01 & 73.28 & 22.01 & 14.45 & 36.61 \\
  CEN    & 42.20 & 32.08 & 61.31 & 31.50 & 21.70 & 50.59 & 46.84 & 36.38 & 67.01 & 20.39 & 12.96 & 34.97 \\
  HTCCN  & 45.39 & 36.58 & --    & 35.63 & 24.90 & --    & 51.94 & 40.32 & --    & 23.46 & 15.18 & --    \\
  HERLN  & 43.94 & 34.62 & 63.44 & 31.33 & 21.93 & 52.01 & --    & --    & --    & --    & --    & --    \\
  TiRGN  & 44.75 & 34.26 & 65.28 & 33.66 & 23.19 & 54.22 & 50.04 & 39.25 & 70.71 & 21.67 & 13.63 & 37.60 \\
  LogCL  & 48.87 & 37.76 & 70.26 & 35.67 & 24.53 & 57.74 & 57.04 & 46.07 & 77.87 & 23.75 & 14.64 & 42.33 \\
  DyMemR & 46.12 & 36.38 & 65.15 & 35.50 & 25.29 & 55.44 & 53.76 & 44.68 & 70.84 & 25.46 & 16.79 & 42.49 \\
  HisRES & \underline{50.48} & \underline{39.57} & \underline{71.09}
         & \underline{37.69} & \underline{26.46} & \underline{59.70}
         & \underline{59.07} & \underline{48.62} & \underline{78.48}
         & \underline{26.58} & \underline{16.90} & \underline{46.31} \\

  \cmidrule(l){1-13}

  \textbf{CHE-TKG}
  & \textbf{51.91} & \textbf{41.37} & \textbf{72.17}
  & \textbf{38.77} & \textbf{27.58} & \textbf{60.95}
  & \textbf{60.37} & \textbf{50.04} & \textbf{79.44}
  & \textbf{27.38} & \textbf{17.73} & \textbf{47.02} \\

  \bottomrule
  \end{tabular}
  }
  \label{tab:main_result}
\end{table*}

\section{Experiments}
\subsection{Experimental Settings}
\textbf{Datasets.} We evaluate our method on four widely used benchmark datasets for TKG extrapolation: ICEWS14s, ICEWS18, ICEWS05-15, and GDELT, adopting the standard data splits provided by prior works \cite{li2021temporal,liang2023learn,chen2024local,zhang2025historically}.

\textbf{Evaluation Protocols.} Following prior work, we evaluate the entity extrapolation task using time-aware filtered metrics, including Mean Reciprocal Rank (MRR) and Hits@$\{1,10\}$, and report all results in percentage form. All results are averaged over three runs with different random seeds.

\textbf{Compared Baselines.} We compare our method with a diverse set of embedding-based baselines, including RE-NET \cite{jin2020recurrent}, RE-GCN \cite{li2021temporal}, RETIA \cite{liu2023retia}, RPC \cite{liang2023learn}, TANGO \cite{han2021learning}, HisMatch \cite{li2022hismatch}, CEN \cite{li2022complex}, HTCCN \cite{chen2024htccn}, HERLN \cite{du2025hawkes}, TiRGN \cite{li2022tirgn}, LogCL \cite{chen2024local}, DyMemR \cite{zhang2024temporal}, and HisRES \cite{zhang2025historically}.

More detailed experimental settings and analyses, including ablation study, sensitivity analysis, analysis of complementary predictive signals, statistical analysis of results, relation prediction, robustness analysis, visualization analysis, complexity analysis, and execution time analysis, are provided in Appendix~\ref{appe:exper}.

\subsection{Main Results}

Table~\ref{tab:main_result} reports the time-aware filtered evaluation results for the entity extrapolation task on four benchmark datasets.
Overall, CHE-TKG consistently outperforms all the strong baselines.
Specifically, relative to the second-best results, CHE-TKG obtains substantial MRR improvements of 2.9\%, 3.2\%, 2.3\%, and 3.1\% on ICEWS14s, ICEWS18, ICEWS05-15, and GDELT, respectively. 
These consistent gains indicate the advantage of our approach, which explicitly separates and jointly models historical evidence and evolutionary dynamics, and further leverages relation decomposition and contrastive alignment to better preserve and exploit complementary predictive signals, thereby improving ranking reliability and prediction accuracy.
Overall, the results verify the effectiveness of our CHE-TKG for TKGR.

\subsection{Ablation Study}

We conduct ablation studies on ICEWS14s and ICEWS18 using MRR and Hits@\{1,10\}.
The results are summarized in Table~\ref{tab:ablation}, which reports the following variants: the full CHE-TKG; CHE-TKG without the evolutionary dynamics graph ({-w/o-}$\mathcal{G}^{D}$); 
without the historical evidence graph ({-w/o-}$\mathcal{G}^{E}$); 
without both contrastive alignment and relation decomposition, while using both graphs ({-w/o-(CoA \& ReD)});
without contrastive alignment, while using both graphs and relation decomposition ({-w/o-CoA}); 
and without disentangled learning ({-w/o-DiL}).

Overall, removing either the historical evidence graph or the evolutionary dynamics graph results in substantial performance degradation, whereas jointly modeling both consistently yields the best results, validating the effectiveness of the proposed collaborative learning framework. 
Excluding relation decomposition or contrastive alignment also results in clear performance drops, highlighting the importance of view-specific relation representations and cross-view alignment for suppressing view-specific noise.
We further conduct a simple coupled learning variant ({-w/o-DiL}), where the historical evidence graph is learned first, followed by the evolutionary dynamics graph. 
This variant surpasses both single-graph counterparts, highlighting the importance of collaborative learning over the two views.
However, it still underperforms compared to the variant that jointly learns from both graphs ({-w/o-(CoA \& ReD)}), and further falls short of the full CHE-TKG. 
This highlights the effectiveness of our approach in better preserving and exploiting complementary predictive signals.
\vspace{-5pt}

\begin{table*}[h]

  \centering
  \setlength{\tabcolsep}{2pt}
  \caption{Ablation results on ICEWS14s and ICEWS18.}

  \begin{tabular}{lcccccc}
  \toprule
  \multirow{2}{*}{Variant} 
  & \multicolumn{3}{c}{ICEWS14s} 
  & \multicolumn{3}{c}{ICEWS18} \\
  \cmidrule(lr){2-4}\cmidrule(lr){5-7}
  & MRR & Hits@1 & Hits@10 
  & MRR & Hits@1 & Hits@10 \\
  \midrule

  CHE-TKG
  & \textbf{51.91} & \textbf{41.37} & \textbf{72.17}
  & \textbf{38.77} & \textbf{27.58} & \textbf{60.95} \\
  \midrule
  
  -w/o-$\mathcal{G}^{D}$
  & 49.54 & 38.82 & 70.00
  & 36.88 & 25.53 & 59.40 \\
  -w/o-$\mathcal{G}^{E}$
  & 50.29 & 39.75 & 70.55
  & 36.98 & 25.81 & 58.92 \\
  -w/o-(CoA \& ReD)
  & 51.37 & 40.91 & 71.50
  & 38.42 & 27.15 & 60.80 \\
  -w/o-CoA
  & 51.70 & 41.19 & 71.73
  & 38.62 & 27.46 & 60.57 \\
  
  -w/o-DiL
  & 51.03 & 40.41 & 71.13 
  & 38.13 & 26.73 & 60.56 \\

  \bottomrule
  \end{tabular}
  \label{tab:ablation}
\end{table*}
\vspace{-5pt}

\begin{table*}[ht!]
  \centering
  \setlength{\tabcolsep}{2pt}
  \caption{Case study of CHE-TKG's entity extrapolation performance on the ICEWS14s dataset. The ground-truth entity is highlighted in bold, and the corresponding prediction scores are provided.}
  \resizebox{1.0\textwidth}{!}{
  \begin{tabular}{l cc cc cc}
  \toprule
  \multirow{2}{*}{Query} 
    & \multicolumn{2}{c}{CHE-TKG-w/$\mathcal{G}^D$} 
    & \multicolumn{2}{c}{CHE-TKG-w/$\mathcal{G}^E$}
    & \multicolumn{2}{c}{CHE-TKG} \\
  \cmidrule(lr){2-3}\cmidrule(lr){4-5}\cmidrule(lr){6-7}
    & Hits@5 Entity & Score & Hits@5 Entity & Score & Hits@5 Entity & Score \\
  \midrule
  
  \multirow{5}{*}{\makecell[l]{
    \makecell[l]{$e_s:Head\_of$ \\ $\_Government\_(India)$} \\
    $r:Praise\_or\_endorse$ \\
    $t:2014-12-01$ }}
    & \makecell[c]{$Member\_of\_$ \\ $Parliament\_(India)$} & 5.1521 & $Citizen\_(India)$ & 4.6465 & $\bm{Employee\_(India)}$ & \textbf{9.5145} \\
    & $Governor\_(India)$               & 5.1514 & \makecell[c]{$Company\_-\_Owner$ \\ $\_or\_Operator\_(India)$} & 4.4339 & \makecell[c]{$Member \_of$ \\ $\_Parliament\_(India)$} & 9.1098 \\
    & $\bm{Employee\_(India)}$          & \textbf{5.1435} & \makecell[c]{$Indigenous\_$ \\ $People\_(India)$} & 4.4003 & $Governor\_(India)$ & 8.9817 \\
    & $Labor\_Union\_(India)$           & 4.4350 & $\bm{Employee\_(India)}$ & \textbf{4.3710} & $Citizen\_(India)$ & 8.7354 \\
    & \makecell[c]{$Administrative$ \\ $\_Body\_(India)$}   & 4.3371 & $Villager\_(India)$ & 4.2844 & $Labor\_Union\_(India)$ & 7.9872 \\
  \midrule
  
  
  \multirow{5}{*}{\makecell[l]{
    $e_s:Afghanistan$ \\
    \makecell[l]{$r:Engage\_in$ \\ $\_negotiation$} \\
    $t:2014-12-25$ }}
    & $Abdullah\_Abdullah$       & 8.0501 & $Iraq$ & 10.2351 & $\bm{Tajikistan}$ & \textbf{16.7835} \\
    & $Raheel\_Sharif$           & 8.0464 & $\bm{Tajikistan}$ & \textbf{9.9402} & $Iraq$ & 15.1064 \\
    & $China$                    & 7.3678 & $Afghanistan$ & 9.1128 & $China$ & 14.8508 \\
    & \makecell[c]{$Refugee\_$ \\ $(Afghanistan)$}   & 7.0583 & $Iran$ & 7.9106 & $Iran$ & 14.1855 \\
    & $\bm{Tajikistan}$          & \textbf{6.8432} & $China$ & 7.4830 & $Afghanistan$ & 12.9165 \\
  
  \bottomrule
  \end{tabular}}
  \label{tab:dual_hits5}
\end{table*}

\subsection{Case Study}

To further demonstrate the effectiveness of CHE-TKG, we conduct a qualitative analysis on two representative facts from the ICEWS14s dataset using different model variants.
The results are reported in Table~\ref{tab:dual_hits5}, where “CHE-TKG-w/$\mathcal{G}^{D}$” and “CHE-TKG-w/$\mathcal{G}^{E}$” denote variants that rely solely on the evolutionary dynamics graph and the historical evidence graph, respectively, and “CHE-TKG” represents the full model.
The ground-truth entity for each query is highlighted in bold, and the top-ranked entities under the Hits@5 setting along with their scores are reported.
From the first example, the ground-truth entity \textit{Employee (India)} appears in the top-5 predictions when using either graph alone, but is not ranked at the top by either variant.
By combining scores from both graph views, CHE-TKG assigns the highest score to the ground-truth entity, leading to a correct prediction. 
These cases demonstrate that the historical evidence graph and the evolutionary dynamics graph provide non-redundant information.
Similar observations can be made for the other example.
Jointly modeling both views enables more accurate and reliable entity extrapolation, highlighting the effectiveness of the proposed collaborative learning framework.

\subsection{Analysis of Complementary Predictive Signals}\label{exp:signal}

We compute pair-level score margins and report the statistics in Table~\ref{tab:statistics}. As shown in the table, across all datasets, both $\bar{\gamma}_D$ and $\bar{\gamma}_E$ are positive, and the pairwise correlation $\rho_{\text{pair}}$ is smaller than 1. These observations empirically validate the existence of complementary predictive signals as characterized in Theorem~\ref{theo:error}. 
Although the empirical margins do not exactly satisfy the identical-marginal assumption, the observed trend is consistent with the theorem.
We further report the pairwise error probabilities $P_{\mathrm{err}}^D$, $P_{\mathrm{err}}^E$, and $P_{\mathrm{err}}^{D+E}$. The results show that
$
P_{\mathrm{err}}^{D+E} < \min\{P_{\mathrm{err}}^D, P_{\mathrm{err}}^E\},
$
which is consistent with the theoretical results in Theorem~\ref{theo:error}.
This reduction in pairwise ranking errors also provides an explanation for the improvement in MRR and Hits@K. For each query, the rank of the ground-truth entity is determined by the number of negative entities that receive higher scores. Therefore, reducing the probability that negative entities outrank the ground-truth entity improves ranking reliability and leads to better MRR and Hits@K.
\vspace{-5pt}

\begin{table}[h]
  \centering
  \setlength{\tabcolsep}{4pt}
  \renewcommand{\arraystretch}{1.1}
  \caption{Statistics on different datasets.}
  \label{tab:statistics}
  
  \begin{tabular}{lccccccc}
  \toprule
  Dataset 
  & $\bar{\gamma}_D$ 
  & $\bar{\gamma}_E$ 
  & $\bar{\gamma}_{D+E}$ 
  & $\rho_{\text{pair}}$ 
  & $P_{\mathrm{err}}^D$ 
  & $P_{\mathrm{err}}^E$ 
  & $P_{\mathrm{err}}^{D+E}$ \\
  \midrule
  
  ICEWS14s   & 7.17 & 5.78 & 12.95 & 0.723 & 0.0206 & 0.0283 & 0.0196 \\
  ICEWS18    & 7.45 & 5.45 & 12.90 & 0.754 & 0.0151 & 0.0186 & 0.0141 \\
  ICEWS05-15 & 8.51 & 5.24 & 13.75 & 0.724 & 0.0187 & 0.0280 & 0.0178 \\
  GDELT      & 6.28 & 3.89 & 10.17 & 0.750 & 0.0163 & 0.0203 & 0.0139 \\
  
  \bottomrule
  \end{tabular}
  \end{table}
  \vspace{-5pt}
  \section{Conclusion}
  In this work, we proposed CHE-TKG, a collaborative dual-view learning framework for TKGR. 
  CHE-TKG explicitly separates historical evidence and evolutionary dynamics into two view-specific structural views, allowing the model to learn complementary predictive signals for future fact prediction. 
  The historical evidence view captures long-term structural regularities, while the evolutionary dynamics view models temporal transitions and recent changes. 
  Relation decomposition and contrastive alignment are further introduced to learn view-specific relation semantics and promote task-relevant complementarity between the two views. 
  We provided a theoretical analysis showing that effective and non-redundant predictive signals can improve ranking reliability by reducing pair-level ranking errors. 
  Experiments on multiple benchmark datasets demonstrate that CHE-TKG achieves state-of-the-art performance in extrapolation settings. 
  Overall, these results highlight the benefits of collaboratively modeling historical evidence and evolutionary dynamics for TKGR.

\bibliography{sample-base}

@inproceedings{liu2022tlogic,
  title={TLogic: Temporal logical rules for explainable link forecasting on temporal knowledge graphs},
  author={Liu, Yushan and Ma, Yunpu and Hildebrandt, Marcel and Joblin, Mitchell and Tresp, Volker},
  booktitle={Proceedings of the AAAI Conference on Artificial Intelligence},
  pages={4120--4127},
  year={2022}
}

@inproceedings{li2021temporal,
  title={Temporal knowledge graph reasoning based on evolutional representation learning},
  author={Li, Zixuan and Jin, Xiaolong and Li, Wei and Guan, Saiping and Guo, Jiafeng and Shen, Huawei and Wang, Yuanzhuo and Cheng, Xueqi},
  booktitle={Proceedings of the ACM SIGIR Conference on Research and Development in Information Retrieval},
  pages={408--417},
  year={2021}
}

@inproceedings{liang2023learn,
  title={Learn from relational correlations and periodic events for temporal knowledge graph reasoning},
  author={Liang, Ke and Meng, Lingyuan and Liu, Meng and Liu, Yue and Tu, Wenxuan and Wang, Siwei and Zhou, Sihang and Liu, Xinwang},
  booktitle={Proceedings of the ACM SIGIR Conference on Research and Development in Information Retrieval},
  pages={1559--1568},
  year={2023}
}

@inproceedings{zhang2025historically,
  title={Historically relevant event structuring for temporal knowledge graph reasoning},
  author={Zhang, Jinchuan and Sun, Ming and Mu, Chong and Zhang, Jinhao and Guo, Quanjiang and Tian, Ling},
  booktitle={IEEE International Conference on Data Engineering (ICDE)},
  pages={3179--3192},
  year={2025},
  organization={IEEE}
}

@article{bordes2013translating,
  title={Translating embeddings for modeling multi-relational data},
  author={Bordes, Antoine and Usunier, Nicolas and Garcia-Duran, Alberto and Weston, Jason and Yakhnenko, Oksana},
  journal={Advances in neural information processing systems},
  year={2013}
}

@inproceedings{leblay2018deriving,
  title={Deriving Validity Time in Knowledge Graph},
  author={Leblay, Julien and Chekol, Melisachew Wudage},
  booktitle={Proceedings of the ACM Web Conference},
  year={2018},
}

@inproceedings{lacroixtensor,
  title={Tensor Decompositions for Temporal Knowledge Base Completion},
  author={Lacroix, Timoth{\'e}e and Obozinski, Guillaume and Usunier, Nicolas},
  booktitle={International Conference on Learning Representations},
  year={2020}
}

@inproceedings{messner2022temporal,
  title={Temporal knowledge graph completion using box embeddings},
  author={Messner, Johannes and Abboud, Ralph and Ceylan, Ismail Ilkan},
  booktitle={Proceedings of the AAAI Conference on Artificial Intelligence},
  pages={7779--7787},
  year={2022}
}

@inproceedings{fang2024transformer,
  title={Transformer-based reasoning for learning evolutionary chain of events on temporal knowledge graph},
  author={Fang, Zhiyu and Lei, Shuai-Long and Zhu, Xiaobin and Yang, Chun and Zhang, Shi-Xue and Yin, Xu-Cheng and Qin, Jingyan},
  booktitle={Proceedings of the ACM SIGIR Conference on Research and Development in Information Retrieval},
  pages={70--79},
  year={2024}
}

@inproceedings{fang2024arbitrary,
  title={Arbitrary Time Information Modeling via Polynomial Approximation for Temporal Knowledge Graph Embedding},
  author={Fang, Zhiyu and Qin, Jingyan and Zhu, Xiaobin and Yang, Chun and Yin, Xu-Cheng},
  booktitle={Proceedings of the Joint International Conference on Computational Linguistics, Language Resources and Evaluation},
  pages={1455--1465},
  year={2024}
}

@inproceedings{liu2025terdy,
  title={Terdy: Temporal relation dynamics through frequency decomposition for temporal knowledge graph completion},
  author={Liu, Ziyang and Wang, Chaokun},
  booktitle={Proceedings of the Annual Meeting of the Association for Computational Linguistics},
  pages={9611--9622},
  year={2025}
}

@inproceedings{ying2024simple,
  title={Simple but Effective Compound Geometric Operations for Temporal Knowledge Graph Completion},
  author={Ying, Rui and Hu, Mengting and Wu, Jianfeng and Xie, Yalan and Liu, Xiaoyi and Wang, Zhunheng and Jiang, Ming and Gao, Hang and Zhang, Linlin and Cheng, Renhong},
  booktitle={Proceedings of the Annual Meeting of the Association for Computational Linguistics},
  pages={11074--11086},
  year={2024}
}

@inproceedings{li2023teast,
  title={Teast: Temporal knowledge graph embedding via archimedean spiral timeline},
  author={Li, Jiang and Su, Xiangdong and Gao, Guanglai},
  booktitle={Proceedings of the Annual Meeting of the Association for Computational Linguistics},
  pages={15460--15474},
  year={2023}
}

@inproceedings{du2025hawkes,
  title={Hawkes based Representation Learning for Reasoning over Scale-free Community-structured Temporal Knowledge Graphs},
  author={Du, Yuwei and Liu, Xinyue and Liang, Wenxin and Zong, Linlin and Zhang, Xianchao},
  booktitle={Proceedings of the International Conference on Computational Linguistics},
  pages={2935--2946},
  year={2025}
}

@inproceedings{chen2024local,
  title={Local-global history-aware contrastive learning for temporal knowledge graph reasoning},
  author={Chen, Wei and Wan, Huaiyu and Wu, Yuting and Zhao, Shuyuan and Cheng, Jiayaqi and Li, Yuxin and Lin, Youfang},
  booktitle={IEEE International Conference on Data Engineering (ICDE)},
  pages={733--746},
  year={2024},
  organization={IEEE}
}

@inproceedings{liu2023retia,
  title={RETIA: relation-entity twin-interact aggregation for temporal knowledge graph extrapolation},
  author={Liu, Kangzheng and Zhao, Feng and Xu, Guandong and Wang, Xianzhi and Jin, Hai},
  booktitle={IEEE International Conference on Data Engineering (ICDE)},
  pages={1761--1774},
  year={2023},
  organization={IEEE}
}

@inproceedings{li2022tirgn,
  title={Tirgn: Time-guided recurrent graph network with local-global historical patterns for temporal knowledge graph reasoning.},
  author={Li, Yujia and Sun, Shiliang and Zhao, Jing},
  booktitle={Proceedings of the International Joint Conference on Artificial Intelligence},
  pages={2152--2158},
  year={2022}
}

@inproceedings{han2021learning,
  title={Learning neural ordinary equations for forecasting future links on temporal knowledge graphs},
  author={Han, Zhen and Ding, Zifeng and Ma, Yunpu and Gu, Yujia and Tresp, Volker},
  booktitle={Proceedings of the Conference on Empirical Methods in Natural Language Processing},
  pages={8352--8364},
  year={2021}
}

@inproceedings{li2022hismatch,
  title={HiSMatch: Historical Structure Matching based Temporal Knowledge Graph Reasoning},
  author={Li, Zixuan and Hou, Zhongni and Guan, Saiping and Jin, Xiaolong and Peng, Weihua and Bai, Long and Lyu, Yajuan and Li, Wei and Guo, Jiafeng and Cheng, Xueqi},
  booktitle={Findings of the Association for Computational Linguistics: EMNLP},
  pages={7328--7338},
  year={2022}
}

@inproceedings{li2022complex,
  title={Complex Evolutional Pattern Learning for Temporal Knowledge Graph Reasoning},
  author={Li, Zixuan and Guan, Saiping and Jin, Xiaolong and Peng, Weihua and Lyu, Yajuan and Zhu, Yong and Bai, Long and Li, Wei and Guo, Jiafeng and Cheng, Xueqi},
  booktitle={Proceedings of the Annual Meeting of the Association for Computational Linguistics},
  pages={290--296},
  year={2022}
}

@inproceedings{sun2022graph,
  title={Graph hawkes transformer for extrapolated reasoning on temporal knowledge graphs},
  author={Sun, Haohai and Geng, Shangyi and Zhong, Jialun and Hu, Han and He, Kun},
  booktitle={Proceedings of the Conference on Empirical Methods in Natural Language Processing},
  pages={7481--7493},
  year={2022}
}

@inproceedings{zhang2023learning,
  title={Learning latent relations for temporal knowledge graph reasoning},
  author={Zhang, Mengqi and Xia, Yuwei and Liu, Qiang and Wu, Shu and Wang, Liang},
  booktitle={Proceedings of the Annual Meeting of the Association for Computational Linguistics},
  pages={12617--12631},
  year={2023}
}

@inproceedings{zhang2023learning_2,
  title={Learning long-and short-term representations for temporal knowledge graph reasoning},
  author={Zhang, Mengqi and Xia, Yuwei and Liu, Qiang and Wu, Shu and Wang, Liang},
  booktitle={Proceedings of the ACM Web Conference},
  pages={2412--2422},
  year={2023}
}

@inproceedings{tang2024editkg,
  title={Editkg: Editing knowledge graph for recommendation},
  author={Tang, Gu and Gan, Xiaoying and Wang, Jinghe and Lu, Bin and Wu, Lyuwen and Fu, Luoyi and Zhou, Chenghu},
  booktitle={Proceedings of the ACM SIGIR Conference on Research and Development in Information Retrieval},
  pages={112--122},
  year={2024}
}

@inproceedings{wang2023mixed,
  title={Mixed-curvature manifolds interaction learning for knowledge graph-aware recommendation},
  author={Wang, Jihu and Shi, Yuliang and Yu, Han and Wang, Xinjun and Yan, Zhongmin and Kong, Fanyu},
  booktitle={Proceedings of the ACM SIGIR Conference on Research and Development in Information Retrieval},
  pages={372--382},
  year={2023}
}

@inproceedings{xu2024retrieval,
  title={Retrieval-augmented generation with knowledge graphs for customer service question answering},
  author={Xu, Zhentao and Cruz, Mark Jerome and Guevara, Matthew and Wang, Tie and Deshpande, Manasi and Wang, Xiaofeng and Li, Zheng},
  booktitle={Proceedings of the ACM SIGIR Conference on Research and Development in Information Retrieval},
  pages={2905--2909},
  year={2024}
}

@inproceedings{ding2024enhancing,
  title={Enhancing complex question answering over knowledge graphs through evidence pattern retrieval},
  author={Ding, Wentao and Li, Jinmao and Luo, Liangchuan and Qu, Yuzhong},
  booktitle={Proceedings of the ACM Web Conference},
  pages={2106--2115},
  year={2024}
}

@inproceedings{atif2023beamqa,
  title={Beamqa: Multi-hop knowledge graph question answering with sequence-to-sequence prediction and beam search},
  author={Atif, Farah and El Khatib, Ola and Difallah, Djellel},
  booktitle={Proceedings of the ACM SIGIR Conference on Research and Development in Information Retrieval},
  pages={781--790},
  year={2023}
}

@inproceedings{liu2022joint,
  title={Joint knowledge graph completion and question answering},
  author={Liu, Lihui and Du, Boxin and Xu, Jiejun and Xia, Yinglong and Tong, Hanghang},
  booktitle={Proceedings of the ACM SIGKDD International Conference on Knowledge Discovery \& Data Mining},
  pages={1098--1108},
  year={2022}
}

@article{shazeer2020glu,
  title={Glu variants improve transformer},
  author={Shazeer, Noam},
  journal={arXiv preprint arXiv:2002.05202},
  year={2020}
}

@inproceedings{xu2020inductive,
  title     = {Inductive Representation Learning on Temporal Graphs},
  author    = {Xu, Da and Ruan, Chuanwei and Korpeoglu, Evren and Kumar, Sushant and Achan, Kannan},
  booktitle = {International Conference on Learning Representations},
  year      = {2020}
}

@dataset{Boschee2015ICEWS,
  author    = {Boschee, Elizabeth and
               Lautenschlager, Jennifer and
               O'Brien, Sean and
               Shellman, Steve and
               Starz, James and
               Ward, Michael},
  title     = {{ICEWS Coded Event Data}},
  year      = {2015},
  publisher = {Harvard Dataverse},
  doi       = {10.7910/DVN/28075},
  url       = {https://doi.org/10.7910/DVN/28075}
}

@inproceedings{leetaru2013gdelt,
  title={GDELT: Global data on events, location, and tone},
  author={Leetaru, Kalev and Schrodt, Philip A},
  booktitle={ISA Annual Convention},
  pages={1--49},
  year={2013}
}

@article{chen2021dacha,
  title={DACHA: A dual graph convolution based temporal knowledge graph representation learning method using historical relation},
  author={Chen, Ling and Tang, Xing and Chen, Weiqi and Qian, Yuntao and Li, Yansheng and Zhang, Yongjun},
  journal={ACM Transactions on Knowledge Discovery from Data},
  pages={1--18},
  year={2021},
  publisher={ACM New York, NY}
}

@inproceedings{chen2024htccn,
  title={Htccn: Temporal causal convolutional networks with hawkes process for extrapolation reasoning in temporal knowledge graphs},
  author={Chen, Tingxuan and Long, Jun and Yang, Liu and Wang, Zidong and Wang, Yongheng and Jin, Xiongnan},
  booktitle={Proceedings of the North American Chapter of the Association for Computational Linguistics},
  pages={4056--4066},
  year={2024}
}

@inproceedings{li2023tr,
  title={Tr-rules: Rule-based model for link forecasting on temporal knowledge graph considering temporal redundancy},
  author={Li, Ningyuan and Haihong, E and Li, Shi and Sun, Mingzhi and Yao, Tianyu and Song, Meina and Wang, Yong and Luo, Haoran},
  booktitle={Findings of the Association for Computational Linguistics: EMNLP},
  pages={7885--7894},
  year={2023}
}

@inproceedings{huang2024confidence,
  title={Confidence is not timeless: Modeling temporal validity for rule-based temporal knowledge graph forecasting},
  author={Huang, Rikui and Wei, Wei and Qu, Xiaoye and Zhang, Shengzhe and Chen, Dangyang and Cheng, Yu},
  booktitle={Proceedings of the Annual Meeting of the Association for Computational Linguistics},
  pages={10783--10794},
  year={2024}
}

@inproceedings{li2025infer,
  title={INFER: A Neural-symbolic Model For Extrapolation Reasoning on Temporal Knowledge Graph},
  author={Li, Ningyuan and Haihong, E and Yao, Tianyu and Hu, Tianyi and Li, Yuhan and Luo, Haoran and Song, Meina and Zhu, Yifan},
  booktitle={International Conference on Learning Representations},
  year={2025}
}

@inproceedings{lee2023temporal,
  title={Temporal Knowledge Graph Forecasting Without Knowledge Using In-Context Learning},
  author={Lee, Dong-Ho and Ahrabian, Kian and Jin, Woojeong and Morstatter, Fred and Pujara, Jay},
  booktitle={Proceedings of the Conference on Empirical Methods in Natural Language Processing},
  pages={544--557},
  year={2023}
}

@inproceedings{liao2024gentkg,
  title={Gentkg: Generative forecasting on temporal knowledge graph with large language models},
  author={Liao, Ruotong and Jia, Xu and Li, Yangzhe and Ma, Yunpu and Tresp, Volker},
  booktitle={Findings of the Association for Computational Linguistics: NAACL},
  pages={4303--4317},
  year={2024}
}

@article{wang2024large,
  title={Large language models-guided dynamic adaptation for temporal knowledge graph reasoning},
  author={Wang, Jiapu and Sun, Kai and Luo, Linhao and Wei, Wei and Hu, Yongli and Liew, Alan W and Pan, Shirui and Yin, Baocai},
  journal={Advances in Neural Information Processing Systems},
  pages={8384--8410},
  year={2024}
}

@article{bai2025g2s,
  title={G2S: A General-to-Specific Learning Framework for Temporal Knowledge Graph Forecasting with Large Language Models},
  author={Bai, Long and Li, Zixuan and Jin, Xiaolong and Guo, Jiafeng and Cheng, Xueqi and Chua, Tat-Seng},
  journal={Findings of the Association for Computational Linguistics: ACL},
  year={2025}
}

@inproceedings{tang2025anre,
  title={AnRe: Analogical Replay for Temporal Knowledge Graph Forecasting},
  author={Tang, Guo and Chu, Zheng and Zheng, Wenxiang and Xiang, Junjia and Li, Yizhuo and Zhang, Weihao and Liu, Ming and Qin, Bing},
  booktitle={Proceedings of the Annual Meeting of the Association for Computational Linguistics},
  pages={4632--4650},
  year={2025}
}

@inproceedings{jin2020recurrent,
  title={Recurrent event network: Autoregressive structure inference over temporal knowledge graphs},
  author={Jin, Woojeong and Qu, Meng and Jin, Xisen and Ren, Xiang},
  booktitle={Proceedings of the 2020 conference on empirical methods in natural language processing (EMNLP)},
  pages={6669--6683},
  year={2020}
}

@inproceedings{xia2024chain,
  title={Chain-of-history reasoning for temporal knowledge graph forecasting},
  author={Xia, Yuwei and Wang, Ding and Liu, Qiang and Wang, Liang and Wu, Shu and Zhang, Xiao-Yu},
  booktitle={Findings of the Association for Computational Linguistics: ACL},
  pages={16144--16159},
  year={2024}
}

@inproceedings{dong2023adaptive,
  title={Adaptive path-memory network for temporal knowledge graph reasoning},
  author={Dong, Hao and Ning, Zhiyuan and Wang, Pengyang and Qiao, Ziyue and Wang, Pengfei and Zhou, Yuanchun and Fu, Yanjie},
  booktitle={Proceedings of the International Joint Conference on Artificial Intelligence},
  pages={2086--2094},
  year={2023}
}

@article{zhang2024temporal,
  title={Temporal knowledge graph reasoning with dynamic memory enhancement},
  author={Zhang, Fuwei and Zhang, Zhao and Zhuang, Fuzhen and Zhao, Yu and Wang, Deqing and Zheng, Hongwei},
  journal={IEEE Transactions on Knowledge and Data Engineering},
  volume={36},
  number={11},
  pages={7115--7128},
  year={2024},
  publisher={IEEE}
}

@inproceedings{pezeshki2021gradient,
  title={Gradient starvation: A learning proclivity in neural networks},
  author={Pezeshki, Mohammad and Kaba, Oumar and Bengio, Yoshua and Courville, Aaron C and Precup, Doina and Lajoie, Guillaume},
  booktitle={Advances in Neural Information Processing Systems},
  pages={1256--1272},
  year={2021}
}

@inproceedings{scimecashortcut,
  title={Which Shortcut Cues Will DNNs Choose? A Study from the Parameter-Space Perspective},
  author={Scimeca, Luca and Oh, Seong Joon and Chun, Sanghyuk and Poli, Michael and Yun, Sangdoo},
  booktitle={International Conference on Learning Representations},
  year={2022}
}

@inproceedings{deng2020dynamic,
  title={Dynamic knowledge graph based multi-event forecasting},
  author={Deng, Songgaojun and Rangwala, Huzefa and Ning, Yue},
  booktitle={Proceedings of the 26th ACM SIGKDD international conference on knowledge discovery \& data mining},
  pages={1585--1595},
  year={2020}
}

@article{zhang2022temporal,
  title={Temporal knowledge graph representation learning with local and global evolutions},
  author={Zhang, Jiasheng and Liang, Shuang and Sheng, Yongpan and Shao, Jie},
  journal={Knowledge-Based Systems},
  volume={251},
  pages={109234},
  year={2022},
  publisher={Elsevier}
}

@article{tang2024dhyper,
  title={DHyper: A recurrent dual hypergraph neural network for event prediction in temporal knowledge graphs},
  author={Tang, Xing and Chen, Ling and Shi, Hongyu and Lyu, Dandan},
  journal={ACM Transactions on Information Systems},
  volume={42},
  number={5},
  pages={1--23},
  year={2024},
  publisher={ACM New York, NY}
}

@article{tang2023gtrl,
  title={GTRL: An entity group-aware temporal knowledge graph representation learning method},
  author={Tang, Xing and Chen, Ling},
  journal={IEEE Transactions on Knowledge and Data Engineering},
  volume={36},
  number={9},
  pages={4707--4721},
  year={2023},
  publisher={IEEE}
}

@article{chen2024decrl,
  title={DECRL: A deep evolutionary clustering jointed temporal knowledge graph representation learning approach},
  author={Chen, Qian and Chen, Ling},
  journal={Advances in Neural Information Processing Systems},
  volume={37},
  pages={55204--55227},
  year={2024}
}

@inproceedings{chenbeyond,
  title={Beyond Entity Correlations: Disentangling Event Causal Puzzles in Temporal Knowledge Graphs},
  author={Chen, Qian and Zhang, Jinyu and Chen, Ling},
  booktitle={International Conference on Learning Representations},
  year={2026}
}

@inproceedings{gastinger2022evaluation,
  title={On the evaluation of methods for temporal knowledge graph forecasting},
  author={Gastinger, Julia and Sztyler, Timo and Sharma, Lokesh and Schuelke, Anett},
  booktitle={NeurIPS Temporal Graph Learning Workshop},
  year={2022}
}
\bibliographystyle{plain}

\appendix
\section{Related Work}\label{appe:related}
\subsection{Differences from Local-Based Methods}
Our method differs fundamentally from prior approaches such as HGLS \cite{zhang2023learning_2}, LogCL \cite{chen2024local}, and HisRES \cite{zhang2025historically}. These methods typically exploit recent snapshots as local information, which is substantially different from the evolutionary dynamics modeled in our framework.
It is important to distinguish local-based modeling from evolutionary dynamics. Local-based modeling captures temporal context from recent snapshots, often via GNNs and recurrent modules, and serves as a general-purpose representation learning component. In contrast, evolutionary dynamics are modeled from a predictive perspective, focusing on task-relevant temporal transitions of entities and relations. 
Moreover, evolutionary dynamics capture logically related events across multiple timestamps, resulting in substantially broader temporal coverage than local-based modeling.
In other words, local-based modeling provides general contextual information, whereas evolutionary dynamics capture predictive signals. They differ fundamentally in their role, modeling level, and objective.

\subsection{Interpolation Reasoning}
Interpolation reasoning focuses on inferring facts at observed past timestamps. Many existing methods aim to capture temporal dependencies through diverse modeling strategies \cite{leblay2018deriving,lacroixtensor,messner2022temporal,fang2024transformer,fang2024arbitrary,liu2025terdy,ying2024simple,li2023teast}. For example, TTransE \cite{leblay2018deriving} extends TransE \cite{bordes2013translating} by modeling time and relations as translation operations, while TComplex \cite{lacroixtensor} captures temporal dependencies while supporting atemporal facts. BoxTE \cite{messner2022temporal} incorporates temporal information via relation-specific transition matrices, and ECEformer \cite{fang2024transformer} models temporal patterns through evolutionary event chains. 
Recently, several approaches have explored modeling TKGs in the complex-valued space.
TCompoundE \cite{ying2024simple} models temporal KGs via relation- and time-specific translation and scaling transformations.
TeRDy \cite{liu2025terdy} decomposes relation embeddings in the frequency domain to capture long- and short-term dynamics.

\subsection{Extrapolation Reasoning}
Beyond embedding-based methods, extrapolation approaches can also be categorized into rule-based and LLM-based methods.

\textbf{Rule-Based TKGR.}
TLogic \cite{liu2022tlogic} extracts temporal logical rules via non-increasing temporal random walks, estimates rule confidence, and derives predictions through noisy-OR aggregation. 
Subsequent methods extend this framework. 
TR-Rules \cite{li2023tr} mines temporal acyclic rules and introduces window-based confidence to mitigate temporal redundancy, while TempValid \cite{huang2024confidence} further incorporates temporal decay into rule confidence estimation.
Unlike TLogic and its variants, INFER \cite{li2025infer} incorporates temporal validity, fact frequency, and embedding information for extrapolation reasoning, while DaeMon \cite{dong2023adaptive} learns continuous and implicit path representations through neural networks without explicitly constructing logical rules.

\textbf{LLM-Based TKGR.}
Recently, several studies explore leveraging the inductive and reasoning capabilities of large language models (LLMs) for TKGR. 
ICL \cite{lee2023temporal} applies in-context learning by prompting LLMs with historical inputs, while GenTKG \cite{liao2024gentkg} and LLM-DA \cite{wang2024large} incorporate temporal logical rules to retrieve and refine historical sequences.
CoH \cite{xia2024chain} leverages LLMs to explore high-order historical chains and combines their semantic reasoning with structural predictions from graph models.
G2S \cite{bai2025g2s} adopts a two-stage framework that combines ID-based quadruples with textual representations, and AnRe \cite{tang2025anre} further enhances temporal reasoning by guiding LLMs with few-shot analogical examples.

We further compare our method with rule-based and LLM-based approaches. As shown in Table~\ref{tab:main-results-extra} on ICEWS14s, ICEWS18, and ICEWS05-15, our method still achieves the best performance. 
It is worth noting that most of these methods are evaluated on these datasets instead of GDELT, likely due to the substantially larger scale of GDELT.
\begin{table*}[t]
  \centering
  \small
  \setlength{\tabcolsep}{4pt}
  \caption{Performance of TKG entity extrapolation on ICEWS14s, ICEWS18, and ICEWS05-15. The time-aware filtered MRR, Hits@1, and Hits@10 metrics are multiplied by 100. The best results are marked in \textbf{bold}.}
  \label{tab:main-results-extra}
  \resizebox{0.92\textwidth}{!}{
  \begin{tabular}{c l *{3}{c} *{3}{c} *{3}{c}}
    \toprule
    \multirow{2}{*}{Type} & \multirow{2}{*}{Model}
    & \multicolumn{3}{c}{ICEWS14s}
    & \multicolumn{3}{c}{ICEWS18}
    & \multicolumn{3}{c}{ICEWS05-15} \\
    \cmidrule(lr){3-5}\cmidrule(lr){6-8}\cmidrule(lr){9-11}
    & & MRR & Hits@1 & Hits@10
      & MRR & Hits@1 & Hits@10
      & MRR & Hits@1 & Hits@10 \\
    \midrule

    \multirow{3}{*}{Path-based}
    & TLogic   & 43.04 & 33.56 & 61.23 & 29.82 & 20.54 & 48.53 & 46.97 & 36.21 & 67.43 \\
    & TR-Rules & 43.32 & 33.96 & 61.17 & 30.41 & 21.10 & 48.92 & 47.64 & 37.06 & 67.57 \\
    & TempValid & 45.78 & 35.50 & 65.06 & 33.50 & 23.91 & 52.33 & 50.31 & 39.46 & 70.55 \\
    & INFER    & 44.46 & 35.03 & 62.31 & 32.22 & 22.39 & 51.52 & 48.73 & 38.32 & 68.48 \\

    \midrule

    \multirow{3}{*}{LLM-based}
    & GenTKG & --    & 36.85 & 53.50 & --    & 24.25 & 42.10 & --    & --    & --    \\
    & LLM-DA & 47.10 & 36.90 & 67.10 & --    & --    & --    & 52.10 & 41.60 & 72.80 \\
    & MESH   & 44.36 & --    & 64.21 & 33.96 & --    & 54.12 & 48.66 & --    & 68.57 \\
    & AnRe   & 47.40 & 36.90 & 65.70 & 35.50 & 26.00 & 56.70 & 50.90 & 39.10 & 69.60 \\
    \midrule

    Dual-views
    & CHE-TKG
    & \textbf{51.91} & \textbf{41.37} & \textbf{72.17}
    & \textbf{38.77} & \textbf{27.58} & \textbf{60.95}
    & \textbf{60.37} & \textbf{50.04} & \textbf{79.44} \\
    
    \bottomrule
  \end{tabular}
  }
\end{table*}

\section{List of Notations}\label{abla:nota}
We summarize the main notations used in this paper in Table~\ref{tab:nota}.
\begin{table}[h]
    \centering
    \caption{Summary of Important Notations}
    \small
    \begin{tabular}{l l}
    \toprule
    \textbf{Notation} & \textbf{Description} \\
    \midrule
    $\mathcal{E}, \mathcal{R}, \mathcal{T}, \mathcal{F}$ 
    & Entity, relation, timestamp, and temporal fact set in a TKG. \\
  
    $\mathcal{G}_t$ 
    & Temporal knowledge graph snapshot at timestamp $t$. \\
  
    $\mathbf{e}_{s,t}, \mathbf{r}_{t}, \mathbf{e}_{o,t}$ 
    & Subject, relation, and object embeddings at timestamp $t$. \\
  
    $\mathbf{E}_t, \mathbf{R}_t$ 
    & Entity and relation embedding matrices at timestamp $t$. \\
  
    $\mathcal{G}^{D}, \mathcal{G}^{E}$ 
    & Evolutionary dynamics and historical evidence graph. \\
  
    $\mathbf{e}^{D}, \mathbf{e}^{E}$ 
    & Entity embeddings of both graphs. \\
  
    $\mathbf{r}^{D}, \mathbf{r}^{E}$ 
    & Relation embeddings of both graphs. \\
  
    $\mathbf{E}^{D}, \mathbf{E}^{E}$ 
    & Entity embedding matrices of both graphs. \\
  
    $\mathbf{R}^{D}, \mathbf{R}^{E}$ 
    & Relation embedding matrices of both graphs. \\
  
    $d$ 
    & Embedding dimensionality. \\
  
    $L$
    & Number of historical snapshots used for modeling. \\
  
    $\mathcal{Q}_{t_q}$
    & Set of prediction queries associated with timestamp $t_q$. \\
    \bottomrule
    \end{tabular}
    \label{tab:nota}
  \end{table}

\section{Analysis of Rule-Based Retrieval}\label{abla:rule}

In this section, we analyze temporal logical rule-based retrieval used to construct the evolutionary dynamics graph. The rules are learned using the standard TLogic \cite{liu2022tlogic} procedure, relying solely on training data.
In practice, we restrict rule retrieval to one-hop rules for efficiency and reliability, and skip rule-based retrieval for relations that are not covered by any mined rule. We report the learned rules on the ICEWS14s, ICEWS18, ICEWS05-15, and GDELT datasets, along with the coverage of retrieved evolutionary dynamics facts on the test set, as shown in Table~\ref{tab:rule_stats}.

\textbf{Rule-Cov.} denotes the fraction of test queries whose query relation is covered by at least one mined temporal rule. 
\textbf{Retrieved-Cov.} denotes the fraction of test queries for which the mined rules retrieve at least one historical fact. 
\textbf{Avg. Retrieved Facts} denotes the average number of retrieved facts per query.
As shown in the table, the majority of test queries are associated with at least one retrieved evolutionary dynamics fact. For queries with zero retrieved facts, the model cannot rely on explicit evolutionary dynamics and instead makes predictions based on latent entity and relation representations. Furthermore, restricting to one-hop rules still allows the model to effectively capture temporal evolution patterns of entities.
\begin{table}[h]
  \centering
  \setlength{\tabcolsep}{4pt}
  \renewcommand{\arraystretch}{1.1}
  \caption{
  Statistics of rule-based retrieval on different datasets.
  }
  \label{tab:rule_stats}
  
  \begin{tabular}{lccc}
  \toprule
  Dataset 
  & Rule-Cov. 
  & Retrieved-Cov. 
  & Avg. Retrieved Facts \\
  \midrule
  
  ICEWS14s   & 100.00 & 95.849 & 8.7331 \\
  ICEWS18    & 99.976 & 97.321 & 7.4332 \\
  ICEWS05-15 & 99.985 & 98.911 & 9.5729 \\
  GDELT      & 99.999 & 99.780 & 7.9673 \\
  
  \bottomrule
  \end{tabular}
  \end{table}

\section{Proofs}\label{proof}
In this section, we provide a proof of Theorem~\ref{theo:error}.

\textit{Proof.} Assume $(\gamma_D,\gamma_E)$ follows a joint Gaussian distribution with identical marginals:
\begin{equation}
\gamma_D, \gamma_E \sim \mathcal{N}(\mu,\sigma^2),
\end{equation}
and correlation coefficient
\begin{equation}
\operatorname{Corr}(\gamma_D,\gamma_E)=\rho<1.
\end{equation}

Define the fused margin:
\begin{equation}
\gamma_{D+E} = \gamma_D + \gamma_E.
\end{equation}
Since a linear combination of jointly Gaussian variables remains Gaussian, we have
\begin{equation}
\gamma_{D+E} \sim \mathcal{N}\left(\mathbb{E}[\gamma_{D+E}], \operatorname{Var}(\gamma_{D+E})\right).
\end{equation}

The expectation is
\begin{equation}
\mathbb{E}[\gamma_{D+E}] = \mathbb{E}[\gamma_D] + \mathbb{E}[\gamma_E] = 2\mu.
\end{equation}

The variance is
\begin{equation}
\operatorname{Var}(\gamma_{D+E}) 
= \operatorname{Var}(\gamma_D) + \operatorname{Var}(\gamma_E) + 2\operatorname{Cov}(\gamma_D,\gamma_E).
\end{equation}

Using $\operatorname{Var}(\gamma_D)=\operatorname{Var}(\gamma_E)=\sigma^2$ and
\begin{equation}
\operatorname{Cov}(\gamma_D,\gamma_E) = \rho\sigma^2,
\end{equation}
we obtain
\begin{equation}
\operatorname{Var}(\gamma_{D+E}) = 2\sigma^2(1+\rho).
\end{equation}

Thus,
\begin{equation}
\gamma_{D+E} \sim \mathcal{N}(2\mu,\, 2\sigma^2(1+\rho)).
\end{equation}

For a single view, the error probability is
\begin{equation}
P(\gamma_D \le 0) 
= P\!\left(\frac{\gamma_D-\mu}{\sigma} \le -\frac{\mu}{\sigma}\right)
= \Phi\!\left(-\frac{\mu}{\sigma}\right),
\end{equation}
where $\Phi(\cdot)$ is the standard normal cumulative distribution function (CDF). Similarly,
\begin{equation}
P(\gamma_E \le 0) = \Phi\!\left(-\frac{\mu}{\sigma}\right).
\end{equation}

For the fused margin,
\begin{equation}
P(\gamma_{D+E} \le 0)
= P\!\left(
\frac{\gamma_{D+E}-2\mu}{\sqrt{2\sigma^2(1+\rho)}}
\le -\frac{2\mu}{\sqrt{2\sigma^2(1+\rho)}}
\right),
\end{equation}
which gives
\begin{equation}
P(\gamma_{D+E} \le 0)
=
\Phi\!\left(
-\frac{2\mu}{\sqrt{2\sigma^2(1+\rho)}}
\right)
=
\Phi\!\left(
-\frac{\sqrt{2}\mu}{\sigma\sqrt{1+\rho}}
\right).
\end{equation}

Since $\rho<1$, we have $1+\rho<2$, which implies
\begin{equation}
\frac{\sqrt{2}\mu}{\sigma\sqrt{1+\rho}} > \frac{\mu}{\sigma}.
\end{equation}
Multiplying both sides by $-1$ reverses the inequality:
\begin{equation}
-\frac{\sqrt{2}\mu}{\sigma\sqrt{1+\rho}} 
< 
-\frac{\mu}{\sigma}.
\end{equation}

By the monotonicity of $\Phi$, it follows that
\begin{equation}
P(\gamma_{D+E} \le 0) < P(\gamma_D \le 0).
\end{equation}
Since the two views are symmetric, we also have
\begin{equation}
P(\gamma_D \le 0) = P(\gamma_E \le 0).
\end{equation}
Thus,
\begin{equation}
P(\gamma_{D+E} \le 0) < P(\gamma_D \le 0) = P(\gamma_E \le 0).
\end{equation}
\qed

\begin{remark}\label{abla:remark}
Suppose
\begin{equation}
\gamma_D \sim \mathcal{N}(\mu_D,\sigma_D^2), \quad
\gamma_E \sim \mathcal{N}(\mu_E,\sigma_E^2),
\end{equation}
with correlation $\rho$. Then
\begin{equation}
\gamma_{D+E} = \gamma_D + \gamma_E
\sim
\mathcal{N}(\mu_D+\mu_E,\; \sigma_D^2+\sigma_E^2+2\rho\sigma_D\sigma_E).
\end{equation}
The corresponding error probabilities are
\begin{equation}
P(\gamma_D \le 0) = \Phi\!\left(-\frac{\mu_D}{\sigma_D}\right), \quad
P(\gamma_E \le 0) = \Phi\!\left(-\frac{\mu_E}{\sigma_E}\right),
\end{equation}
and
\begin{equation}
P(\gamma_{D+E} \le 0)
=
\Phi\!\left(
-\frac{\mu_D+\mu_E}{\sqrt{\sigma_D^2+\sigma_E^2+2\rho\sigma_D\sigma_E}}
\right).
\end{equation}

Define the signal-to-noise ratios as
\begin{equation}
\mathrm{SNR}_D=\frac{\mu_D}{\sigma_D}, 
\quad
\mathrm{SNR}_E=\frac{\mu_E}{\sigma_E},
\end{equation}
and
\begin{equation}
\mathrm{SNR}_{D+E}
=
\frac{\mu_D+\mu_E}
{\sqrt{\sigma_D^2+\sigma_E^2+2\rho\sigma_D\sigma_E}}.
\end{equation}
Since $\Phi(\cdot)$ is monotonically increasing, if
\begin{equation}
\mathrm{SNR}_{D+E} > \max\{\mathrm{SNR}_D,\mathrm{SNR}_E\},
\end{equation}
then
\begin{equation}
P(\gamma_{D+E}\le 0) < P(\gamma_D\le 0),
\qquad
P(\gamma_{D+E}\le 0) < P(\gamma_E\le 0).
\end{equation}
Thus, in the general Gaussian case, dual-view fusion reduces the ranking error probability when the fused margin achieves a higher signal-to-noise ratio than either single-view margin.
\end{remark}

\section{Extra Experiments}\label{appe:exper}

\begin{table*}[t!]
  \centering
  \caption{Statistics of datasets}
  \resizebox{1.0\textwidth}{!}{
  \begin{tabular}{lccccccc}
    \toprule
  Dataset & Entities $|\mathcal{E}|$ & Relations $|\mathcal{R}|$ & Training Facts & Validation Facts & Testing Facts & Timestamps $|\mathcal{T}|$ & Time granularity \\
  \hline
  ICEWS14s   & 7,128  & 230 & 74,845  & 8,514  & 7,371  & 365 & 24 hours \\
  ICEWS18    & 23,033 & 256 & 373,018 & 45,995 & 49,545 & 304 & 24 hours \\
  ICEWS05-15 & 10,094 & 251 & 368,868 & 46,302 & 46,159 & 4017 & 24 hours \\
  GDELT      & 7,691  & 240 & 1,734,399 & 238,765 & 305,241 & 2976 & 15 mins \\
  \bottomrule
  \end{tabular}
  \label{tab:dataset_statistics_horizontal}
  }
\end{table*}

\subsection{Experimental Settings}\label{appe:setting}
\textbf{Datasets.}
We evaluate our method on four widely used benchmark datasets for TKG extrapolation: ICEWS14s, ICEWS18, ICEWS05-15, and GDELT.
The first three datasets are derived from the Integrated Crisis Early Warning System (ICEWS) \cite{Boschee2015ICEWS}, while GDELT is collected from the Global Database of Events, Language, and Tone \cite{leetaru2013gdelt}.
Following prior work, all datasets are split chronologically into training,
validation, and test sets with proportions of 80\%, 10\%, and 10\%, respectively.
Detailed statistics of the datasets are reported in Table~\ref{tab:dataset_statistics_horizontal}.
Following recent extrapolation-based TKG studies \cite{zhang2025historically,chen2024local,zhang2023learning}, we do not report results on datasets with coarse temporal granularity (e.g., WIKI and YAGO), which are less suitable for modeling fine-grained temporal dynamics.

\textbf{Evaluation Protocols.}
Following prior work \cite{li2021temporal,chen2024local,zhang2025historically,gastinger2022evaluation}, we evaluate the entity extrapolation task using time-aware filtered metrics, including Mean Reciprocal Rank (MRR) and Hits@$\{1,10\}$, and report all results in percentage form. Higher values of MRR and Hits@$k$ indicate better performance, which can be formulated as:  
\begin{equation}\label{equ:MRR}
  MRR = \frac{1}{2*|test|} \sum_{f=(e_s,r,e_o,\tau) \in test} (\frac{1}{k_{f,o}} + \frac{1}{k_{f,s}}),
\end{equation}
\begin{equation}\label{equ:Hits}
  Hits@k = \frac{1}{2*|test|} \sum_{f=(e_s,r,e_o,\tau) \in test} (\mathbbm{1}_{k_{f,o} \leq k} + \mathbbm{1}_{k_{f,s} \leq k}),
\end{equation}
where $k_{f,s}$ and $k_{f,o}$ denote the ranking of subject entity $e_s$ and object entity $e_o$, respectively. $\mathbbm{1}_{condition}$ is 1 if $condition$ holds and 0 otherwise. 

\textbf{Compared Baselines.}
\begin{itemize}
  \item \textbf{RE-NET} \cite{jin2020recurrent} models event sequences in an autoregressive manner, using RNNs to encode historical events and neighbor aggregation to capture structural information.
  
  \item \textbf{RE-GCN} \cite{li2021temporal} employs relation-aware GCNs to model structural dependencies at each timestamp, combined with recurrent units for temporal evolution.
  
  \item \textbf{RETIA} \cite{liu2023retia} constructs dual hyper-relational subgraphs to jointly aggregate entity and relation information with bidirectional interactions.
  
  \item \textbf{RPC} \cite{liang2023learn} models relational correlations and captures temporal periodic patterns via relation and periodic correspondence modules.
  
  \item \textbf{TANGO} \cite{han2021learning} models continuous-time dynamics using neural ODEs with relational graph convolutions.
  
  \item \textbf{HiSMatch} \cite{li2022hismatch} adopts a dual-encoder framework to match query-related and candidate-related historical structures.
  
  \item \textbf{CEN} \cite{li2022complex} captures evolutionary patterns of different lengths via CNNs with curriculum learning.
  
  \item \textbf{HTCCN} \cite{chen2024htccn} combines causal convolutions with Hawkes processes to model temporal dependencies and event intensities.
  
  \item \textbf{HERLN} \cite{du2025hawkes} integrates Hawkes processes with structural properties such as community and long-tail distributions.
  
  \item \textbf{TiRGN} \cite{li2022tirgn} jointly models local temporal dependencies, global historical patterns, and periodicity via hybrid encoders.

  \item \textbf{LogCL} \cite{chen2024local} leverages contrastive learning to integrate local and global historical information with query-aware attention.

  \item \textbf{DyMemR} \cite{zhang2024temporal} employs a dynamic memory module to store and retrieve important historical facts.

  \item \textbf{HisRES} \cite{zhang2025historically} models multi-granularity temporal interactions and selectively aggregates query-relevant historical information.
  \end{itemize}
Results for HGLS and L$^2$TKG are omitted due to differences in experimental settings.
All baseline results are taken from prior work and are reported on the same datasets and under the same filtering protocol for fair comparison.
Specifically, the results of RE-NET, RE-GCN, RETIA, RPC, CEN, TiRGN, LogCL and HisRES are taken from HisRES \cite{zhang2025historically}. The result of TANGO is taken from LogCL \cite{chen2024local}. The results of HisMatch, HTCCN, HERLN, and DyMemR are taken from their original papers. 

\paragraph{\textbf{Implementation Settings.}}
All experiments are conducted on an NVIDIA RTX P40 GPU.
The embedding dimension is set to 200, with a dropout rate of 0.2 for all layers.
Models are trained for up to 50 epochs using early stopping based on validation MRR.
We use the Adam optimizer with a learning rate of 0.001 and a weight decay of $1\times10^{-5}$.
The historical length in the spatio-temporal initialization module is fixed to 3.
The GNN encoder consists of two layers.
The relation prediction coefficient $\alpha$ is set to 0.7.
Hyperparameters, including the contrastive alignment weight $\mu$, temperature $\gamma$, number of GAT layers, and rule graph length, are tuned via grid search on the validation set.
For ICEWS14s, ICEWS18, ICEWS05-15, and GDELT, $\mu$ is set to 0.2, 0.01, 0.15, and 0.3, $\gamma$ to 0.3, 0.03, 0.2, and 0.25, the number of GAT layers to 3, 2, 2, and 2, and the rule graph length to 10, 8, 10, and 8, respectively.
Following prior work \cite{li2021temporal, li2022tirgn, liang2023learn, chen2024local, zhang2025historically}, static KG information is integrated into the ICEWS datasets. We further apply a two-phase forward propagation process to both graph encoders \cite{chen2024local, zhang2025historically}, handling raw and inverse query sets separately. As a result, historical evidence graphs and evolutionary dynamics graphs are constructed independently for each query type, preventing potential data leakage. 

\subsection{Ablation Study}

\begin{table}[t]
  \centering
  \small
  \setlength{\tabcolsep}{4pt}
  \caption{Ablation results on ICEWS14s and ICEWS18.}
  \label{tab:ablation2}

  \begin{tabular}{lcccccc}
  \toprule
  \multirow{2}{*}{Variant} 
  & \multicolumn{3}{c}{ICEWS14s} 
  & \multicolumn{3}{c}{ICEWS18} \\
  \cmidrule(lr){2-4}\cmidrule(lr){5-7}
  & MRR & Hits@1 & Hits@10 
  & MRR & Hits@1 & Hits@10 \\
  \midrule

  CHE-TKG
  & \textbf{51.91} & \textbf{41.37} & \textbf{72.17}
  & \textbf{38.77} & \textbf{27.58} & \textbf{60.95} \\
  \midrule

  -w/o-(CoA \& ReD)
  & 51.37 & 40.91 & 71.50
  & 38.42 & 27.15 & 60.80 \\
  -w/o-$\mathcal{G}^{E}$
  & 50.29 & 39.75 & 70.55
  & 36.98 & 25.81 & 58.92 \\
  -w/Simp
  & 44.91 & 34.17 & 66.32
  & 33.82 & 23.05 & 55.10 \\
  -w/EnD
  & 51.76 & 41.24 & 71.62
  & 38.39 & 27.14 & 60.78 \\
  -w/o-TE
  & 51.35 & 40.77 & 71.68
  & 38.55 & 27.26 & 61.09 \\
  -w/cos
  & 51.38 & 41.01 & 71.51
  & 38.60 & 27.32 & 60.94 \\
  \midrule

  -w/o-init
  & 44.82 & 34.51 & 64.61
  & 34.78 & 24.00 & 56.18 \\
  -w/SiD
  & 51.04 & 40.56 & 71.22
  & 38.13 & 26.76 & 60.73 \\
  -w/WF
  & 51.46 & 40.87 & 71.73
  & 38.55 & 27.20 & 61.06 \\

  \bottomrule
  \end{tabular}
\end{table}
In this subsection, we conduct additional ablation studies to further validate the effectiveness of our design. The results are summarized in Table~\ref{tab:ablation2}, which reports multiple variants: the full CHE-TKG; without both contrastive alignment and relation decomposition while retaining both graphs ({-w/o-(CoA \& ReD)}); without the historical evidence graph ({-w/o-}$\mathcal{G}^{E}$); replacing the rule-based retrieval strategy with a simple heuristic for constructing $\mathcal{G}^{D}$ ({-w/Simp}); introducing entity decomposition analogous to relation decomposition ({-w/EnD}); removing the time encoder in the evolutionary dynamics encoder ({-w/o-TE}); replacing the proposed time encoder with cosine encoding ({-w/cos}); removing spatio-temporal initialization ({-w/o-init}); using a single decoder ({-w/SiD}); and applying adaptive weighted score fusion ({-w/WF}).

As shown in the table, we evaluate a rule-free heuristic for constructing $\mathcal{G}^{D}$ ({-w/Simp}) by selecting the  recent same-subject facts involving the query subject. Compared with {-w/o-$\mathcal{G}^{E}$}, this naive retrieval strategy leads to a substantial performance drop, indicating that simple temporal retrieval is insufficient for capturing informative evolutionary dynamics.
Introducing entity decomposition ({-w/EnD}) results in slight performance degradation, suggesting that relation semantics benefit more from explicit decomposition, while entity decomposition may disrupt shared contextual information.
We further analyze temporal encoding strategies in the evolutionary dynamics encoder. Removing the time encoder ({-w/o-TE}) significantly degrades performance, and replacing it with cosine encoding ({-w/cos}) yields inferior results, demonstrating the importance of explicitly modeling temporal relevance for extrapolation.
Removing the spatio-temporal initialization ({-w/o-init}) also leads to a significant performance drop. This suggests that both historical evidence and evolutionary dynamics lack sufficient general contextual information, and that spatio-temporal initialization plays a crucial role in TKG reasoning.
Compared with {-w/o-(CoA \& ReD)}, using a single decoder ({-w/SiD}) results in degraded performance, indicating that dual decoders enable more effective integration of conditionally non-redundant information at the prediction stage.
Finally, compared with the full CHE-TKG, applying adaptive weighted score fusion ({-w/WF}) does not improve performance, and even leads to slightly inferior results. This may be due to the increased model complexity, which introduces additional optimization difficulty without effectively enhancing predictive signals.

\subsection{Sensitivity Analysis}\label{sec:sensitive}

We experiment on ICEWS14s and ICEWS18 to analyze the effects of key hyperparameters in CHE-TKG, including the upper bound $N$ on the number of historical facts for constructing the evolutionary dynamics graph and the contrastive alignment weight $\mu$. Fig.~\ref{fig:overall_length} shows the impact of $N$ on performance. The best results are achieved at $N=10$ on ICEWS14s and $N=8$ on ICEWS18. 
When $N$ is too small, limited historical information hinders accurate prediction, while overly large $N$ introduces noisy histories and increases computational cost. The smaller optimal $N$ on ICEWS18 can be attributed to its denser event distribution and more stable evolutionary patterns.
Fig.~\ref{fig:overall_mu} illustrates the effect of $\mu$, which controls contrastive alignment strength. Optimal performance is obtained at $\mu=0.2$ on ICEWS14s and $\mu=0.01$ on ICEWS18, indicating that the optimal alignment strength is dataset-dependent. 
In particular, ICEWS14s benefits from stronger alignment due to sparser events and less stable dynamics, whereas ICEWS18 requires weaker alignment given its denser and more stable temporal structure.

\begin{figure}[h]
  \centering
  \subcaptionbox{ICEWS14s\label{fig:left_length}}[0.48\linewidth]{%
    \includegraphics[width=\linewidth]{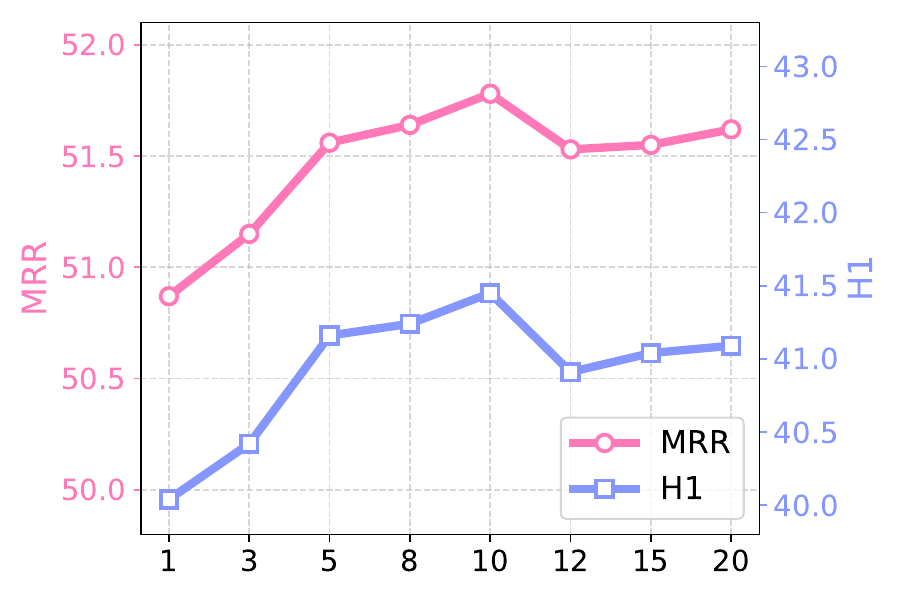}
  }\hspace{0.0\linewidth}
  \subcaptionbox{ICEWS18\label{fig:right_length}}[0.48\linewidth]{%
    \includegraphics[width=\linewidth]{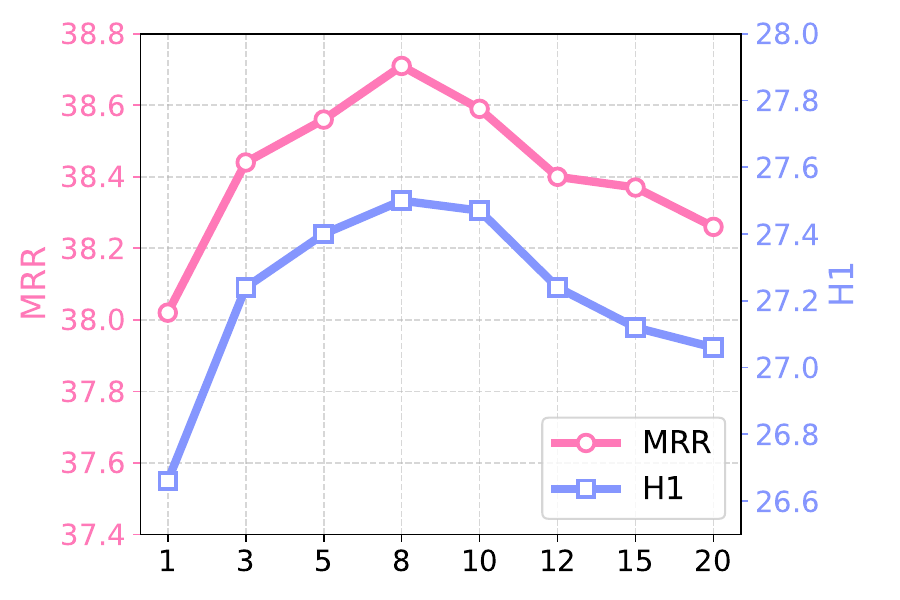}
  }
  \vspace{-5pt}
  \caption{Effect of the upper bound $N$ on the evolutionary dynamics graph.}
  \label{fig:overall_length}
\end{figure}

\begin{figure}[h]
  \centering
  \subcaptionbox{ICEWS14s\label{fig:left_mu}}[0.48\linewidth]{%
    \includegraphics[width=\linewidth]{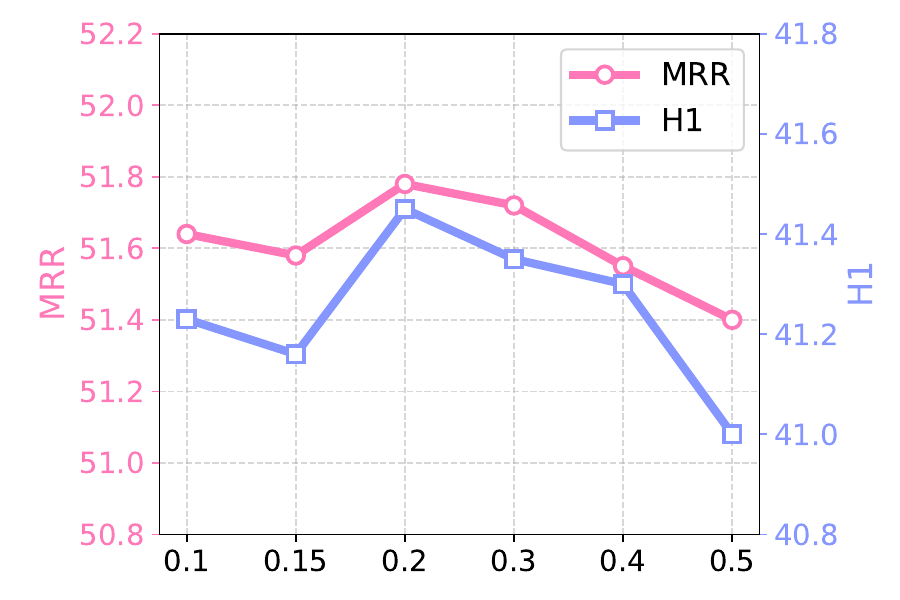}
  }\hspace{0.0\linewidth}
  \subcaptionbox{ICEWS18\label{fig:right_mu}}[0.48\linewidth]{%
    \includegraphics[width=\linewidth]{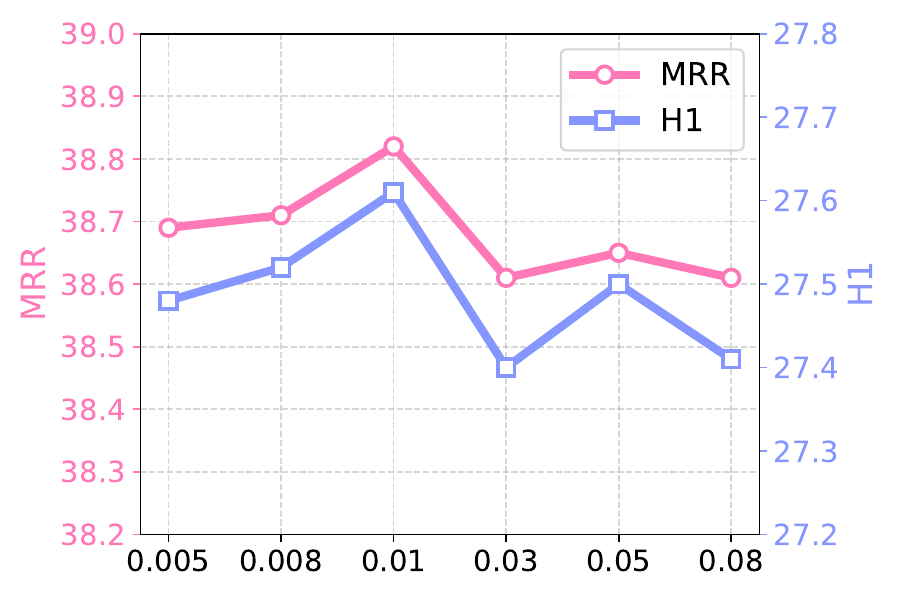}
  }
  \vspace{-5pt}
  \caption{Effect of the contrastive alignment weight $\mu$.}
  \label{fig:overall_mu}
\end{figure}

\subsection{Analysis of Complementary Predictive Signals}

In this subsection, we further report the signal-to-noise ratio (SNR) across different views, as shown in Table~\ref{abla:statistics}. 
Specifically, on ICEWS14s, ICEWS18, and ICEWS05-15, we observe that 
$
\mathrm{SNR}_{D+E} > \max\{\mathrm{SNR}_D, \mathrm{SNR}_E\},
$
which is consistent with the condition proposed in Remark~\ref{abla:remark}. 

On the GDELT dataset, we find that $\mathrm{SNR}_{D+E} > \mathrm{SNR}_E$ while $\mathrm{SNR}_{D+E} \approx \mathrm{SNR}_D$. 
This may be attributed to the large scale of GDELT, where the margin distribution deviates from the Gaussian assumption. 

Although the margin distribution in practice does not always strictly satisfy the sufficient SNR condition under the general Gaussian setting, the empirically measured pairwise ranking error still shows that score fusion effectively reduces the ranking error probability.
\begin{table}[h]
  \centering
  \setlength{\tabcolsep}{3pt}
  \caption{Statistics on different datasets.}
  \label{abla:statistics}
  
  \begin{tabular}{lcccccccccc}
  \toprule
  Dataset 
  & $\bar{\gamma}_D$ 
  & $\bar{\gamma}_E$ 
  & $\bar{\gamma}_{D+E}$ 
  & $\rho_{\text{pair}}$ 
  & $P_{\mathrm{err}}^D$ 
  & $P_{\mathrm{err}}^E$ 
  & $P_{\mathrm{err}}^{D+E}$
  & $SNR_{D}$
  & $SNR_{E}$
  & $SNR_{D+E}$ \\
  \midrule
  
  ICEWS14s   & 7.17 & 5.78 & 12.95 & 0.723 & 0.0206 & 0.0283 & 0.0196 & 2.23 & 2.01 & 2.23 \\
  ICEWS18    & 7.45 & 5.45 & 12.90 & 0.754 & 0.0151 & 0.0186 & 0.0141 & 2.49 & 2.23 & 2.53 \\
  ICEWS05-15 & 8.51 & 5.24 & 13.75 & 0.724 & 0.0187 & 0.0280 & 0.0178 & 2.58 & 2.17 & 2.59 \\
  GDELT      & 6.28 & 3.89 & 10.17 & 0.750 & 0.0163 & 0.0203 & 0.0139 & 2.36 & 1.96 & 2.34 \\
  
  \bottomrule
  \end{tabular}
  \end{table}

\subsection{Statistical Analysis of Results}\label{appe:stat}

We report the standard deviation of evaluation metrics across three runs to assess the stability of our method, with all values reported in percentage form.
On ICEWS14s, the standard deviations of MRR, Hits@1, and Hits@10 are 0.11, 0.13, and 0.13, respectively. 
On ICEWS18, the corresponding values are 0.12, 0.13, and 0.18. 
On ICEWS05-15, they are 0.10, 0.11, and 0.03. 
On GDELT, they are 0.04, 0.04, and 0.01.
These results demonstrate that CHE-TKG achieves stable performance across different datasets.

\subsection{Relation Prediction}
In Eq.~\ref{eq:rel_pred}, relation prediction is incorporated as an auxiliary objective to improve entity prediction. We further report the relation prediction performance of CHE-TKG on the ICEWS14s and ICEWS18 datasets. 
The compared baselines include RE-NET \cite{jin2020recurrent}, Glean \cite{deng2020dynamic}, RE-GCN \cite{li2021temporal}, DACHA \cite{chen2021dacha}, TiRGN \cite{li2022tirgn}, EvoExplore \cite{zhang2022temporal}, GTRL \cite{tang2023gtrl}, DHyper \cite{tang2024dhyper}, DECRL \cite{chen2024decrl}, and HEDRA \cite{chenbeyond}, whose results are all taken from HEDRA \cite{chenbeyond}. 
As shown in the results, our method achieves the second-best performance in terms of MRR and Hits@1 on ICEWS14s, and consistently ranks second across all metrics on ICEWS18. It is worth noting that the current state-of-the-art methods are specifically designed for relation prediction, while our model is not tailored for this task. Nevertheless, CHE-TKG achieves competitive results, suggesting that it captures both entity and relational semantics effectively.

\begin{table}[t]
  \centering
  \small
  \setlength{\tabcolsep}{4pt}
  \caption{Performance comparison on ICEWS14s and ICEWS18. The best results are in \textbf{bold} and the second-best are \underline{underlined}.}
  \label{tab:icews}
  \begin{tabular}{l ccc ccc}
  \toprule
  \multirow{2}{*}{Approach} 
  & \multicolumn{3}{c}{ICEWS14s} 
  & \multicolumn{3}{c}{ICEWS18} \\
  \cmidrule(lr){2-4} \cmidrule(lr){5-7}
  & MRR & Hits@1 & Hits@10 
  & MRR & Hits@1 & Hits@10 \\
  \midrule
  
  RE-NET     & 38.53 & 22.53 & \underline{74.05} & 39.63 & 23.55 & \textbf{75.66} \\
  Glean      & 36.07 & 22.17 & 64.75 & 35.15 & 22.02 & 64.24 \\
  RE-GCN     & 40.85 & 28.15 & 68.53 & 40.68 & 27.01 & 69.51 \\
  DACHA      & 40.69 & 27.28 & 68.44 & 40.80 & 27.83 & 69.21 \\
  TiRGN      & 41.28 & 29.52 & 70.66 & 42.26 & 30.19 & 73.90 \\
  
  EvoExplore & 28.11 & 13.97 & 57.67 & 29.82 & 18.50 & 58.01 \\
  GTRL       & 38.57 & 27.36 & 66.35 & 38.43 & 27.48 & 67.82 \\
  DHyper     & 41.71 & 29.37 & 69.32 & 42.84 & 29.96 & 70.82 \\
  DECRL      & 42.90 & 30.49 & 70.01 & 43.36 & 30.64 & 71.12 \\
  HEDRA      & \textbf{47.86} & \textbf{35.28} & \textbf{75.65} & \textbf{46.77} & \textbf{33.66} & \underline{75.64} \\
  
  \midrule
  CHE-TKG    & \underline{45.07} & \underline{32.74} & 72.56 
             & \underline{45.29} & \underline{32.34} & \underline{73.92} \\
  \bottomrule
  \end{tabular}
  \end{table}

\subsection{Robustness Analysis}

\begin{figure}[h]
  \centering
  \subcaptionbox{ICEWS14s\label{fig:left_robust}}[0.48\linewidth]{%
    \includegraphics[width=\linewidth]{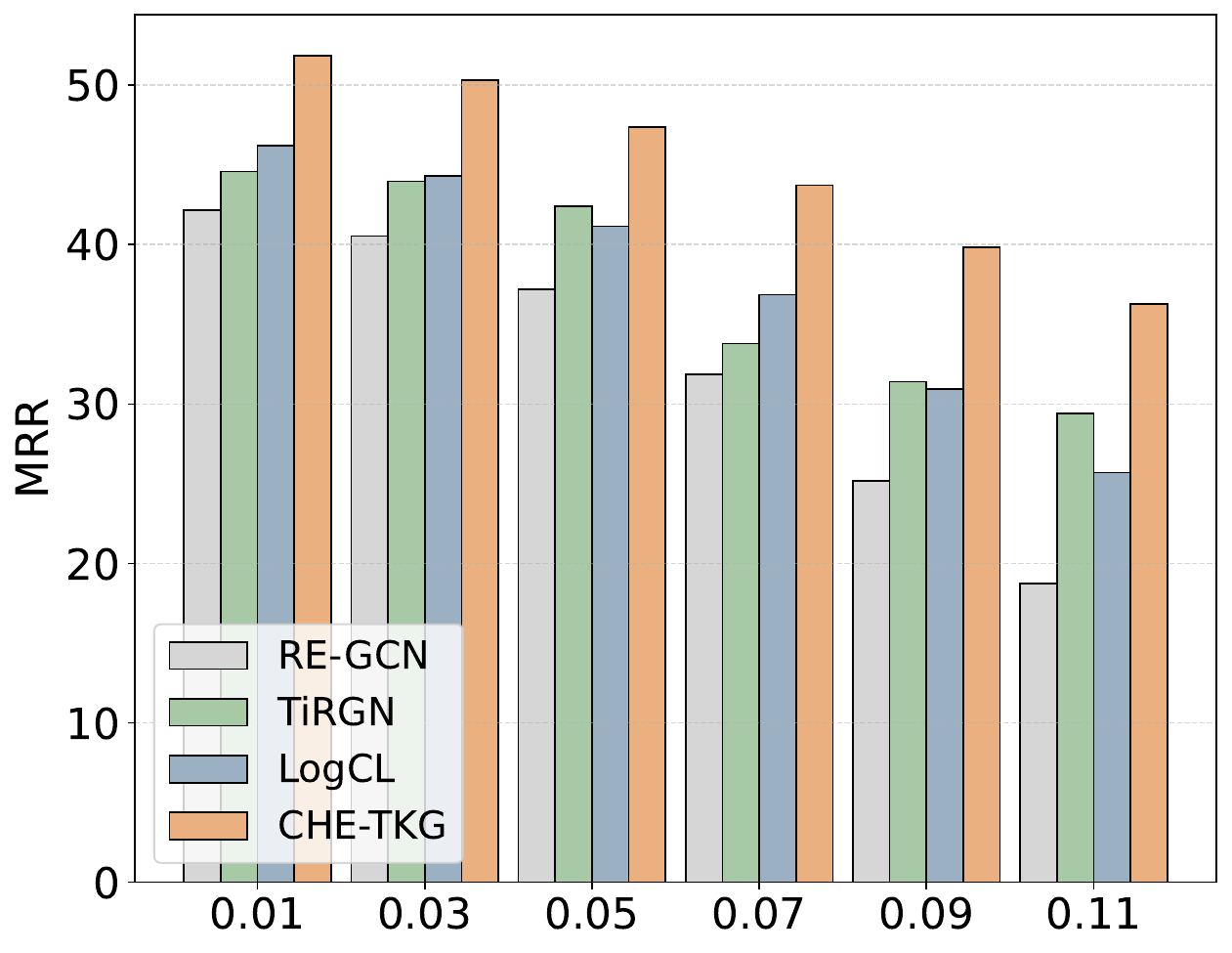}
  }\hspace{0.0\linewidth}
  \subcaptionbox{ICEWS18\label{fig:right_robust}}[0.48\linewidth]{%
    \includegraphics[width=\linewidth]{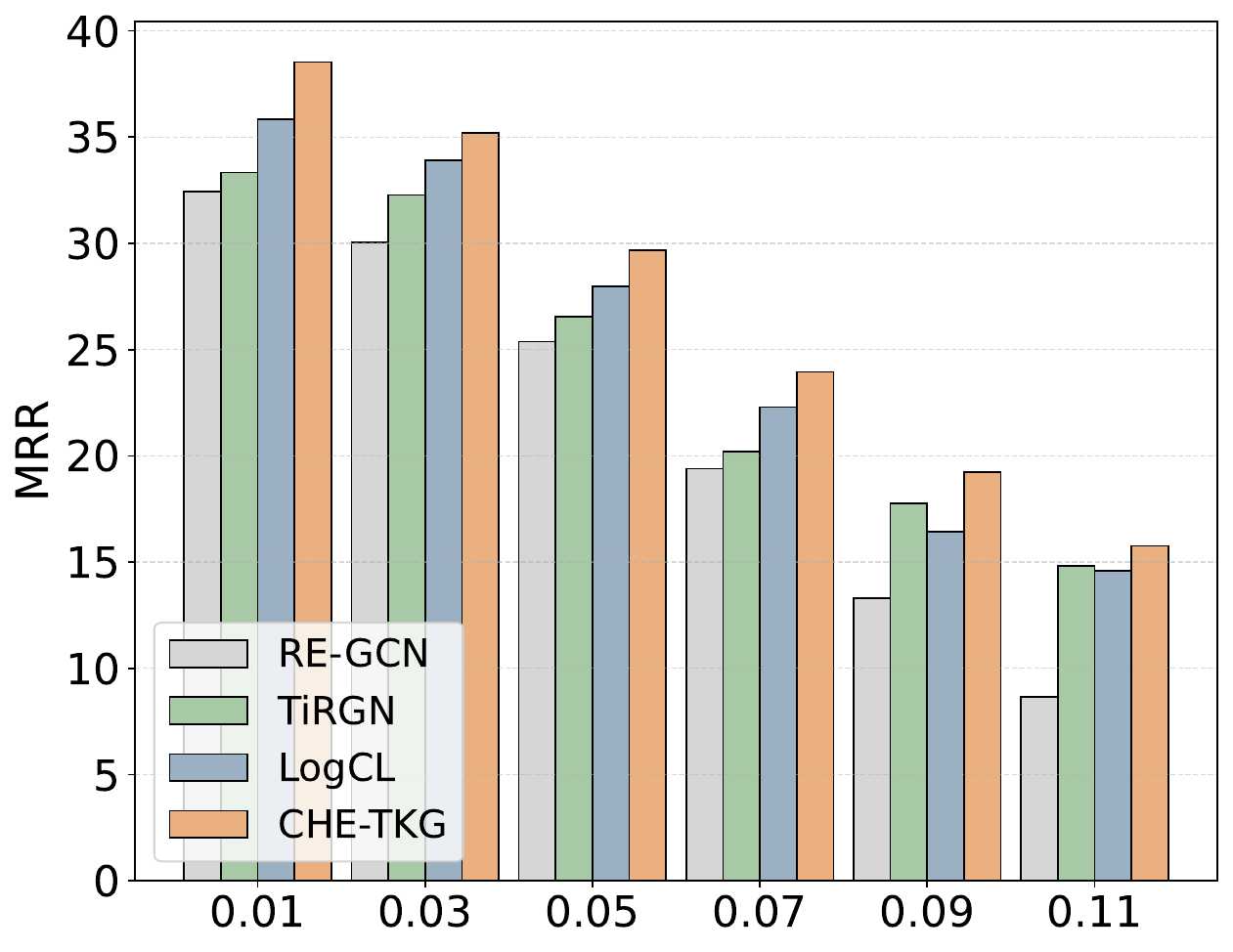}
  }
  \vspace{-5pt}
  \caption{Robustness analysis under Gaussian noise perturbations.}
  \label{fig:overall_robust}
\end{figure}

We evaluate robustness by injecting Gaussian noise with varying intensities into the initial entity embeddings at inference time. 
We compare CHE-TKG against strong open-source baselines, including RE-GCN\footnote{\url{https://github.com/Lee-zix/RE-GCN}}
\cite{li2021temporal}, TiRGN\footnote{\url{https://github.com/Liyyy2122/TiRGN}}
\cite{li2022tirgn}, and LogCL\footnote{\url{https://github.com/WeiChen3690/LogCL}}
\cite{chen2024local}.
 The results are shown in Fig.~\ref{fig:overall_robust}.
On ICEWS14s, CHE-TKG demonstrates the strongest robustness, exhibiting substantially smaller performance degradation as noise intensity increases. 
When the noise level reaches 0.11, CHE-TKG experiences an MRR drop of 30.04\%, compared to 55.54\% for RE-GCN and 44.37\% for LogCL, indicating greater resilience to perturbations in entity representations.
As analyzed in Section~\ref{sec:sensitive}, the evolutionary dynamics in ICEWS18 are relatively more stable, leading to smaller marginal robustness gains of CHE-TKG than those observed on ICEWS14s.

\subsection{Visualization Analysis}

\begin{figure*}[t]
  \centering
  
  \subcaptionbox{Entity}[0.49\textwidth]{%
    \includegraphics[width=\linewidth]{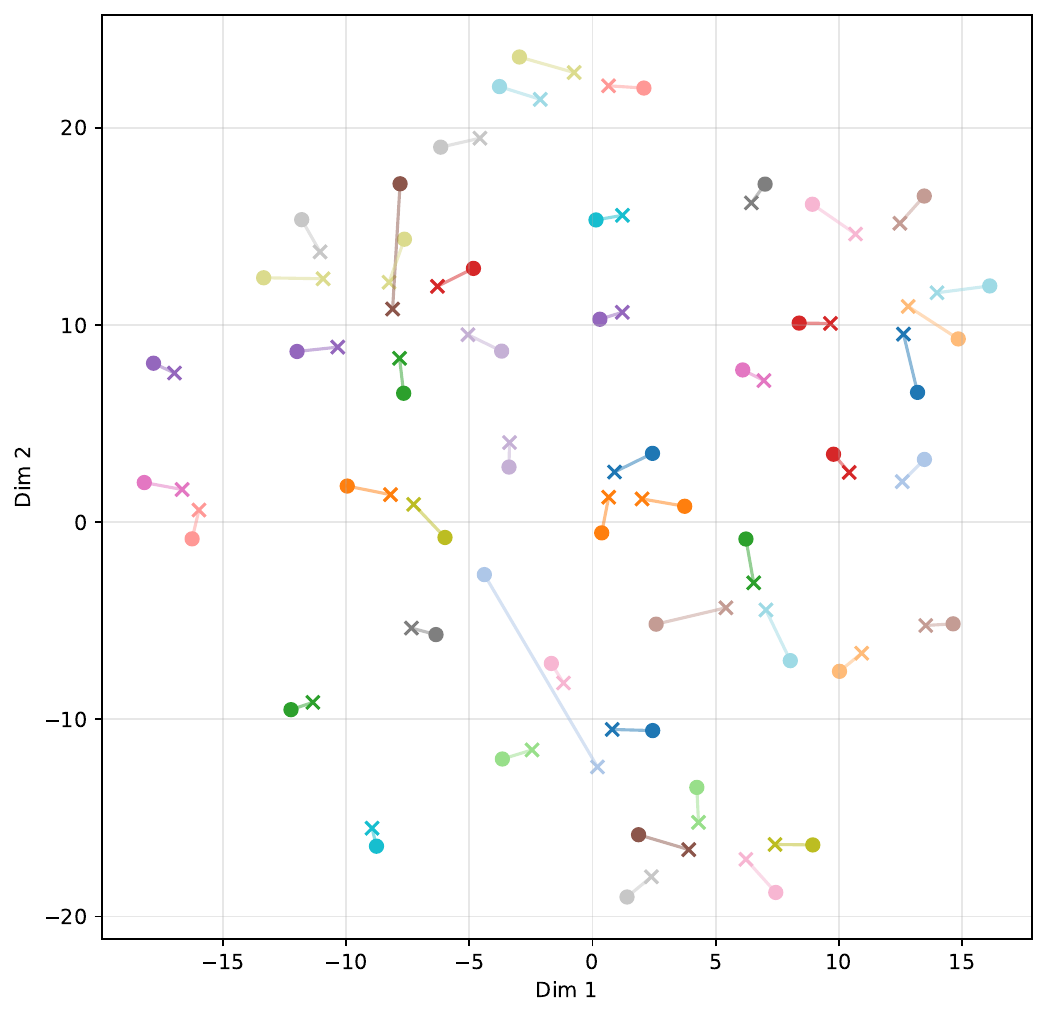}
  }
  \subcaptionbox{Entity}[0.49\textwidth]{%
    \includegraphics[width=\linewidth]{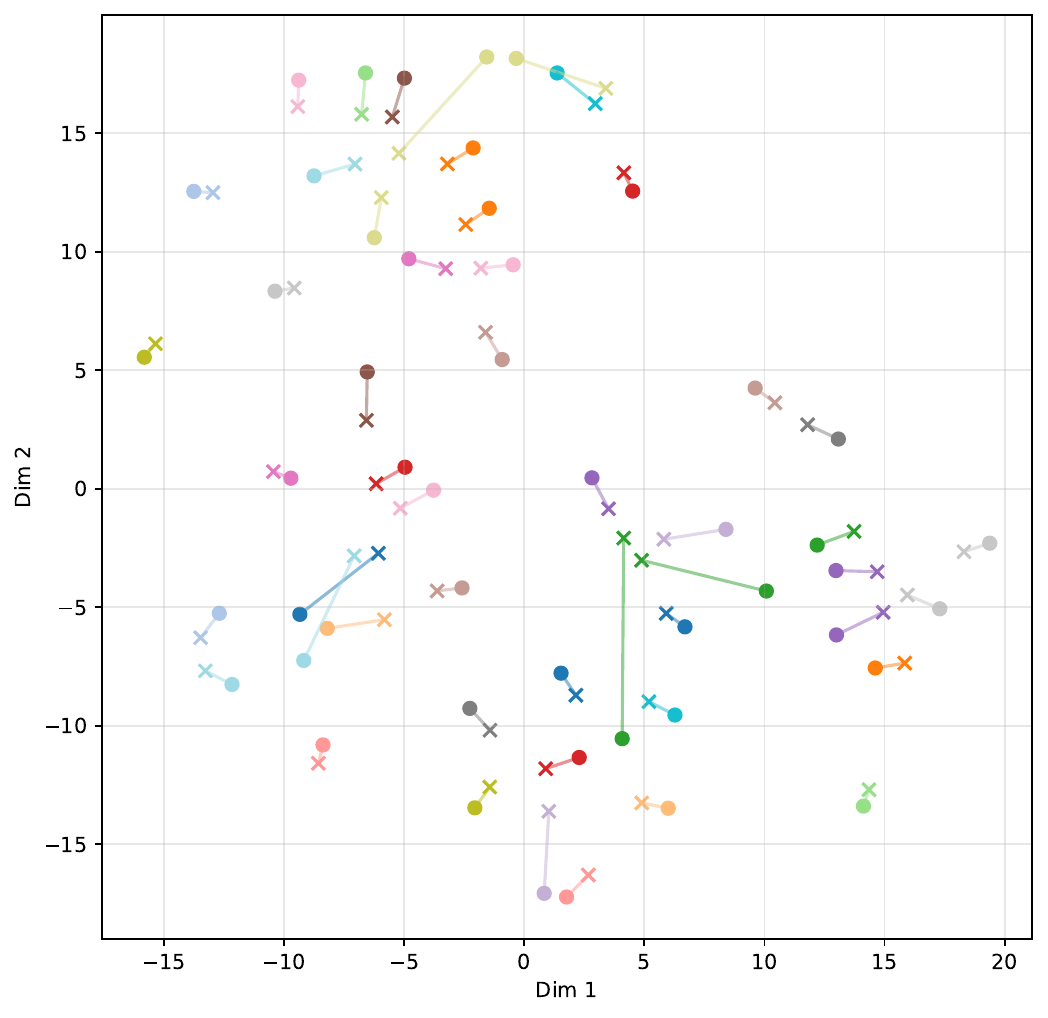}
  }

  \vspace{4pt}

  \subcaptionbox{Relation}[0.49\textwidth]{%
    \includegraphics[width=\linewidth]{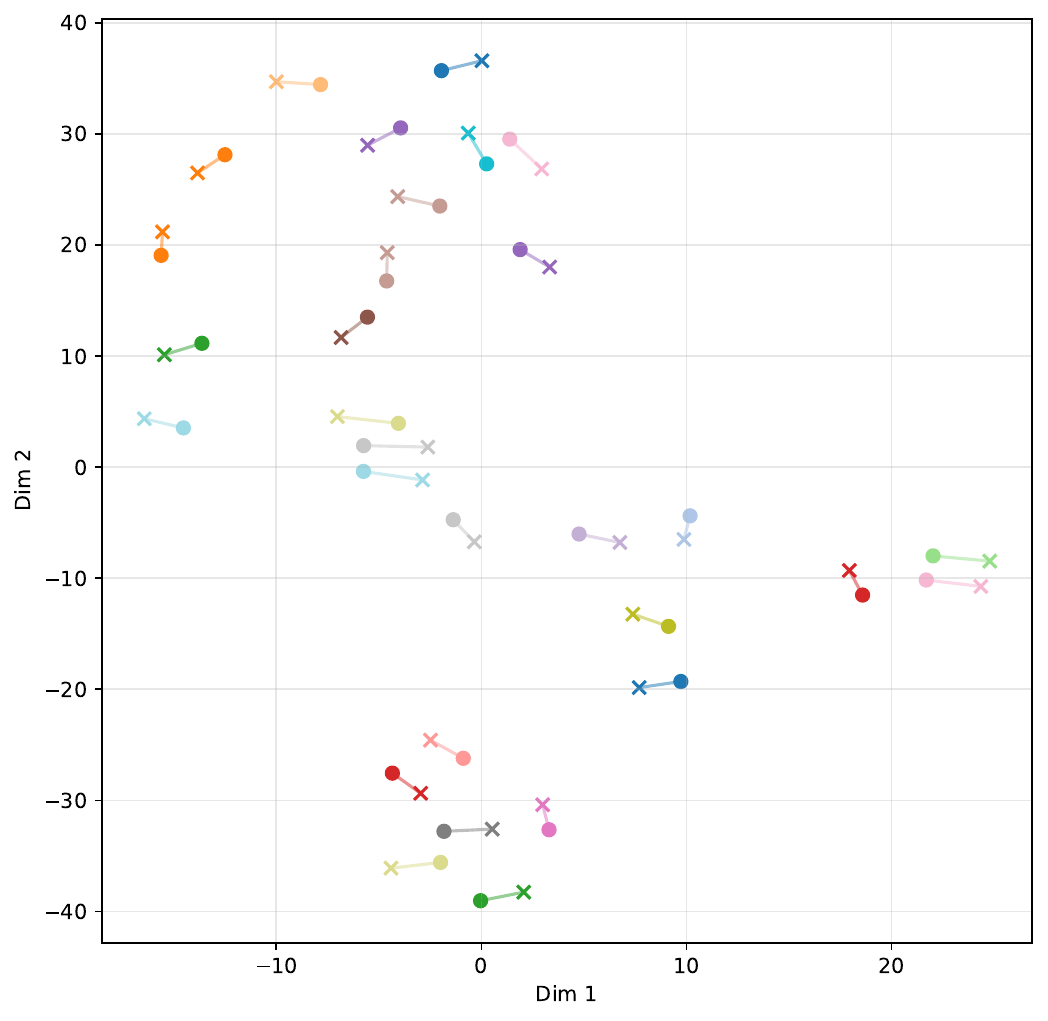}
  }
  \subcaptionbox{Relation}[0.49\textwidth]{%
    \includegraphics[width=\linewidth]{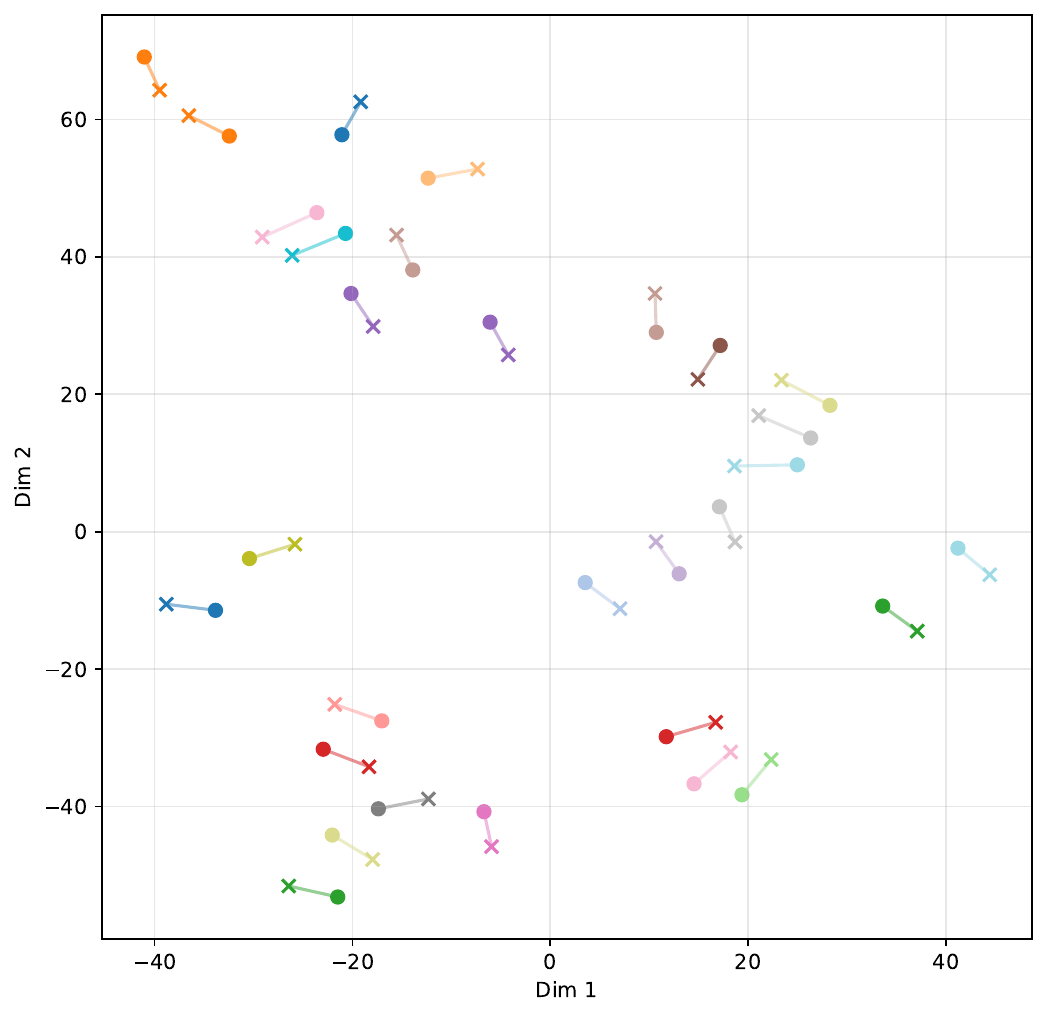}
  }

  \caption{t-SNE visualization of entity and relation embeddings under different settings.}
  \label{fig:tsne}
\end{figure*}

In this subsection, we present a visualization analysis of the entity and relation representations learned by the full framework, which incorporates both relation decomposition and contrastive alignment. 
The results are illustrated in Fig.~\ref{fig:tsne}.
Specifically, we apply t-SNE to visualize entity and relation embeddings from the historical evidence and evolutionary dynamics views on the ICEWS14s test set across different timestamps. We randomly select 50 entity pairs and 30 relation pairs for visualization. The markers “x” and “o” denote embeddings from the historical evidence view and the evolutionary dynamics view, respectively.
As shown in the figure, most embeddings exhibit noticeable yet moderate shifts across the two views, indicating that the representations are not redundant. At the same time, the overall spatial structures remain largely consistent, suggesting the presence of shared information.
These observations imply that the two views capture non-redundant yet correlated signals, which aligns with our assumption of conditional non-redundancy. Moreover, the structured shifts, rather than random displacements, indicate that the differences between the two views are meaningful rather than noise.

\begin{figure}[t]
  \centering
  
  \subcaptionbox{ICEWS14s\label{fig:icews14s}}[0.48\linewidth]{%
    \includegraphics[width=\linewidth]{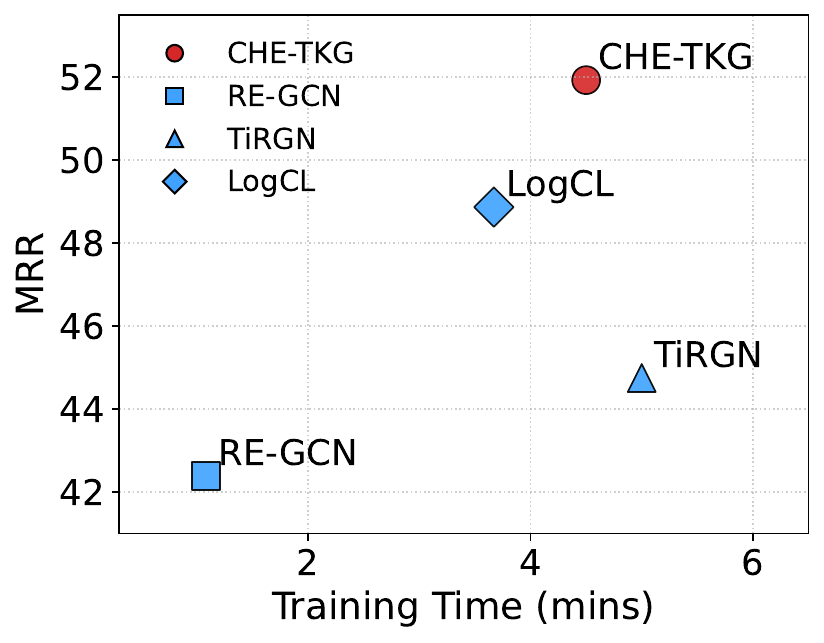}
  }\hspace{0.02\linewidth}
  \subcaptionbox{ICEWS18\label{fig:icews18}}[0.48\linewidth]{%
    \includegraphics[width=\linewidth]{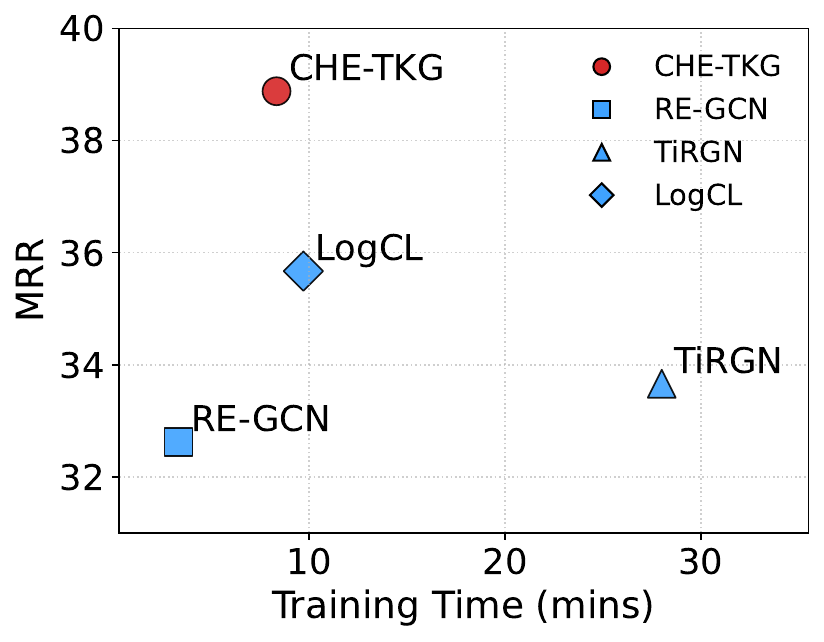}
  }
  
  \vspace{4pt}
  
  \subcaptionbox{ICEWS05-15\label{fig:icews0515}}[0.48\linewidth]{%
    \includegraphics[width=\linewidth]{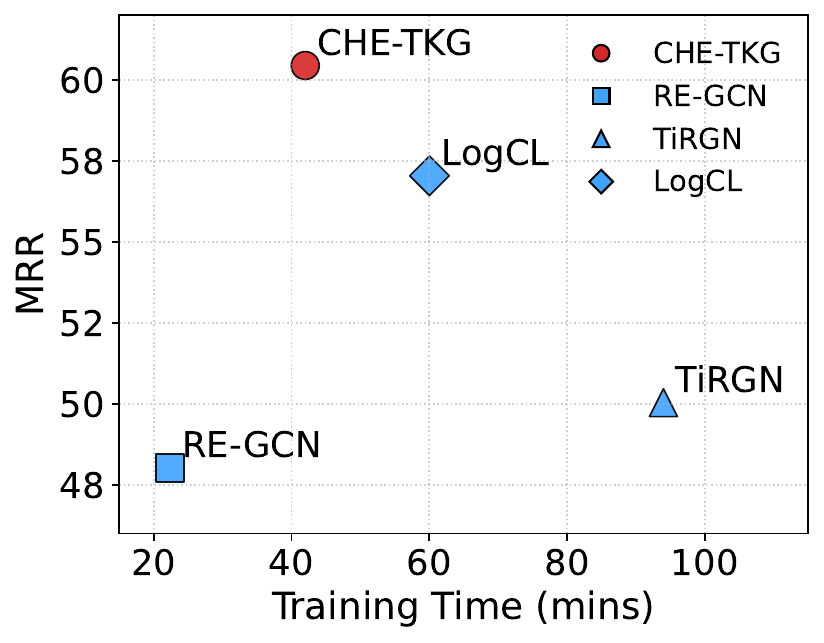}
  }\hspace{0.02\linewidth}
  \subcaptionbox{GDELT\label{fig:gdelt}}[0.48\linewidth]{%
    \includegraphics[width=\linewidth]{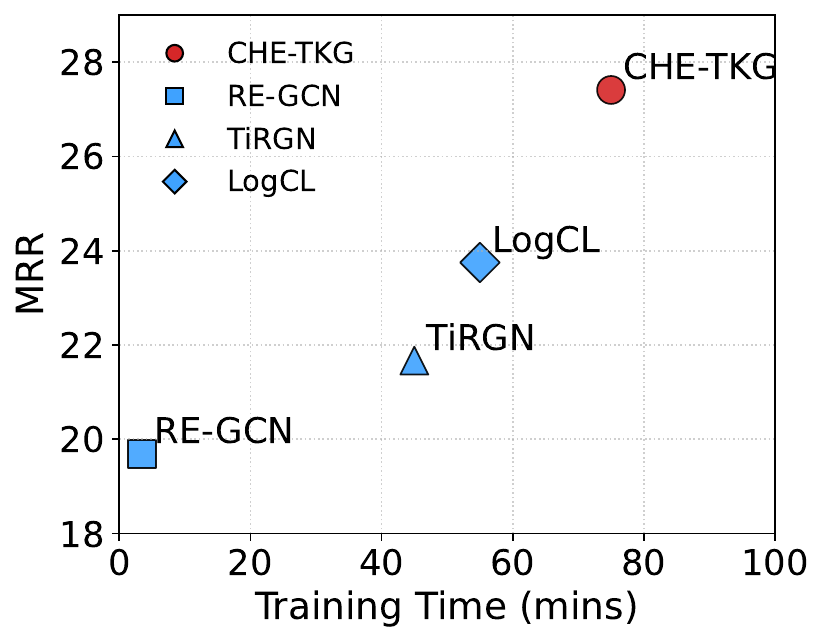}
  }

  \vspace{-5pt}
  \caption{Efficiency–performance trade-off comparisons.}
  \label{fig:overall_time}
\end{figure}

\subsection{Complexity Analysis}
We analyze the computational complexity of CHE-TKG from different components, omitting the embedding dimension $d$ for clarity.
For the historical evidence graph construction, the complexity is $O(|\mathcal{T}||\mathcal{F}_1|)$.
For the evolutionary dynamics graph construction, we adopt temporal logical rule learning based on TLogic, with worst-case complexity $O(|\mathcal{R}| n l D b)$, where $n$ is the number of sampled walks, $l$ is the rule length, $D$ is the maximum node degree, and $b$ is the number of sampled body groundings. Since $l$ is a small constant, this term can be simplified as $O(|\mathcal{R}| n D b)$.
For the temporal logical rule-based retrieval, the complexity is $O(\min(wk, N))$, which can be written as $O(N)$.
For the spatio-temporal initialization module, the complexity is $O(L (|\mathcal{E}| + |\mathcal{R}| + C_{gc}))$, where $C_{gc}$ denotes the cost of graph convolution.
For the relation decomposition module, since each relation is processed independently and the embedding dimension is omitted, the complexity is $O(|\mathcal{R}|)$.
For the collaborative learning stage, the HEGE and EDGE operate on query-specific subgraphs, yielding a complexity of $O(|\mathcal{E}| + |\mathcal{F}_1| + |\mathcal{F}_2|)$.
For the contrastive alignment, the complexity is $O(|\mathcal{Q}_{t_q}|)$.
Overall, the computational complexity of CHE-TKG is
\begin{equation}
  O(L (|\mathcal{E}| + |\mathcal{R}| + C_{gc}) + |\mathcal{R}| n D b + |\mathcal{T}||\mathcal{F}_1| + N + |\mathcal{E}| + |\mathcal{R}| + |\mathcal{F}_1| + |\mathcal{F}_2| + |\mathcal{Q}_{t_q}|).
\end{equation}

\begin{figure}[t]
  \centering
  \includegraphics[width=0.8\linewidth]{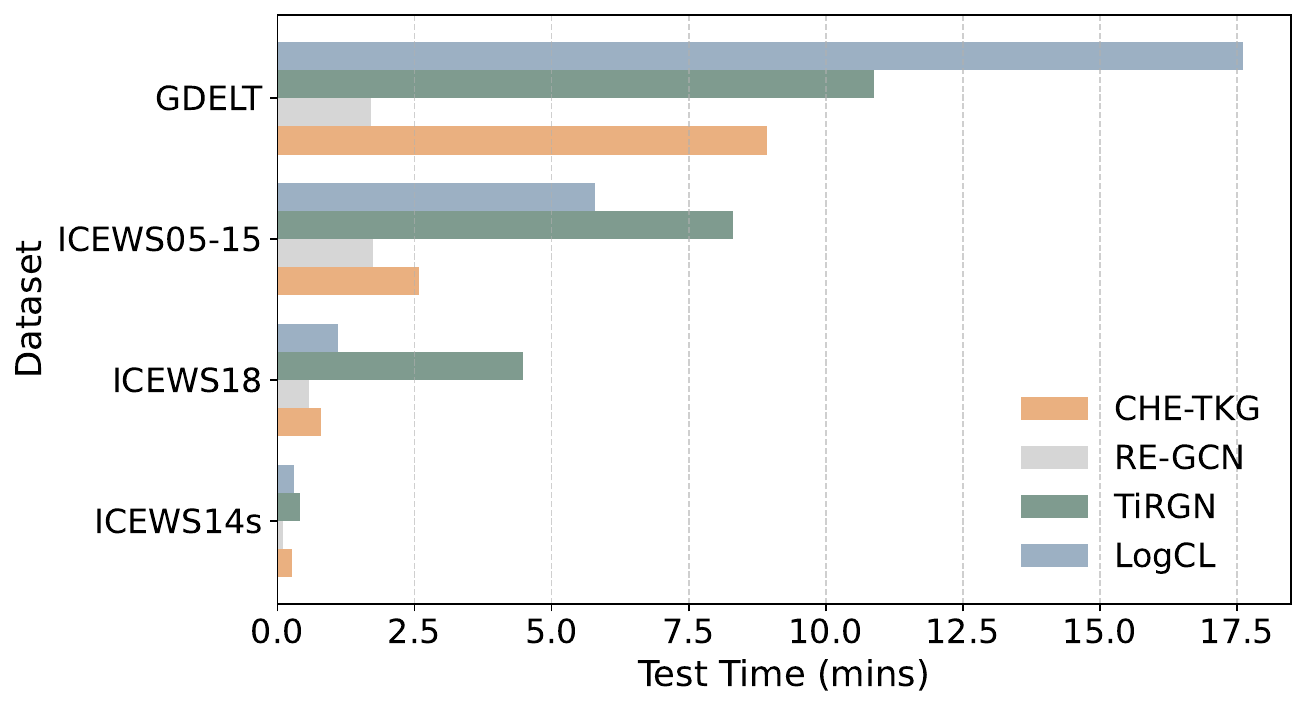}
  \caption{Test-time efficiency comparison.}
  \label{fig:test_time}
\end{figure}

\subsection{Execution Time Analysis}

We analyze the efficiency and predictive performance of CHE-TKG against strong open-source baselines, including RE-GCN \cite{li2021temporal}, TiRGN \cite{li2022tirgn}, and LogCL \cite{chen2024local}, as shown in Fig.~\ref{fig:overall_time}.
Here, training time is measured as the time per epoch.
On all datasets, CHE-TKG achieves a favorable balance between efficiency and performance. 
Compared with RE-GCN, CHE-TKG incurs moderately higher training time but consistently yields substantially better MRR. 
In contrast, more complex methods such as TiRGN and LogCL exhibit significantly higher training costs but underperform CHE-TKG.
These results demonstrate that CHE-TKG provides an effective trade-off between computational efficiency and extrapolation accuracy. 
Despite encoding both historical evidence and evolutionary dynamics graphs, efficient spatio-temporal initialization with a limited number of recent snapshots enables CHE-TKG to achieve strong performance without excessive overhead.
On the GDELT dataset, our method incurs slightly higher training time than all baselines, likely due to the use of GAT.
In addition, CHE-TKG incurs only slightly higher testing time than the simple baseline RE-GCN, but is more efficient than TiRGN and LogCL across ICEWS14s, ICEWS18, ICEWS05-15, and GDELT, as shown in Fig.~\ref{fig:test_time}.

\section{Additional Details on Initialization}
Following prior work~\cite{li2022hismatch,zhang2025historically,chen2024local}, we incorporate temporal encoding in the initialization stage to model temporal evolution and periodic patterns.

Given the relative timestamp $\tau = t_q - t_i$, we compute a cosine-based temporal encoding:
\begin{equation}
\phi(\tau)=\cos\!\left(w_\tau \tau + b_\tau\right),
\end{equation}
where $w_\tau$ and $b_\tau$ are learnable parameters.

The temporal encoding is concatenated with the entity embedding and projected as:
\begin{equation}
\mathbf{e}_t = W_0 \bigl[ \mathbf{e}_i \,\Vert\, \phi(\tau) \bigr],
\end{equation}
where $[\cdot\Vert\cdot]$ denotes concatenation and $W_0 \in \mathbb{R}^{d \times 2d}$ is a linear projection matrix.

The resulting representation $\mathbf{e}_t$ is then used as the input to the GNN in the initialization stage.

\section{Limitations and Broader Impacts}
\subsection{Limitations}\label{appe:limita}
Although our method improves entity extrapolation performance on benchmark datasets by explicitly modeling historical evidence and evolutionary dynamics, it still has several limitations. In particular, it remains unclear which neighborhoods in the historical evidence graph and the evolutionary dynamics graph contribute most to TKGR, limiting the interpretability of CHE-TKG and calling for further investigation.
\subsection{Broader Impacts}\label{appe:impact}
CHE-TKG provides a framework for predicting future facts in temporal knowledge graphs, which can support future-oriented reasoning in applications such as event forecasting, recommendation, risk monitoring, and decision support. 
By jointly modeling historical evidence and evolutionary dynamics, it may help identify plausible future links and provide useful signals for downstream analysis.
However, predictions from CHE-TKG should be used with caution, especially in high-stakes domains. 
Since the model relies on historical data, it may inherit biases, noise, or incompleteness from the observed knowledge graph. 
Potential misuse, such as inferring sensitive relationships or supporting surveillance-oriented applications, should also be considered. 
Therefore, CHE-TKG should be used as a decision-support tool with appropriate human oversight, data governance, and ethical safeguards.

\newpage
\section*{NeurIPS Paper Checklist}

\begin{enumerate}

\item {\bf Claims}
    \item[] Question: Do the main claims made in the abstract and introduction accurately reflect the paper's contributions and scope?
    \item[] Answer: \answerYes{} 
    \item[] Justification: The abstract and introduction include claims that accurately reflect the paper's contributions and scope.

\item {\bf Limitations}
    \item[] Question: Does the paper discuss the limitations of the work performed by the authors?
    \item[] Answer: \answerYes{} 
    \item[] Justification: The limitations of this work are discussed in Appendix~\ref{appe:limita}.

\item {\bf Theory assumptions and proofs}
    \item[] Question: For each theoretical result, does the paper provide the full set of assumptions and a complete (and correct) proof?
    \item[] Answer: \answerYes{} 
    \item[] Justification: Theorem~\ref{theo:error} provides the full assumptions and correct derivations, with the proof included in Appendix~\ref{proof}. We also empirically verify the assumptions in Section~\ref{exp:signal}.

    \item {\bf Experimental result reproducibility}
    \item[] Question: Does the paper fully disclose all the information needed to reproduce the main experimental results of the paper to the extent that it affects the main claims and/or conclusions of the paper (regardless of whether the code and data are provided or not)?
    \item[] Answer: \answerYes{} 
    \item[] Justification: The paper provides all necessary information to reproduce the main experimental results.

\item {\bf Open access to data and code}
    \item[] Question: Does the paper provide open access to the data and code, with sufficient instructions to faithfully reproduce the main experimental results, as described in supplemental material?
    \item[] Answer: \answerYes{} 
    \item[] Justification: The paper provides anonymized code, including training and evaluation code, environment setup, execution commands, and the code, environment, and commands used for graph construction.

\item {\bf Experimental setting/details}
    \item[] Question: Does the paper specify all the training and test details (e.g., data splits, hyperparameters, how they were chosen, type of optimizer) necessary to understand the results?
    \item[] Answer: \answerYes{} 
    \item[] Justification: The paper provides the training and test details in Appendix~\ref{appe:setting}.

\item {\bf Experiment statistical significance}
    \item[] Question: Does the paper report error bars suitably and correctly defined or other appropriate information about the statistical significance of the experiments?
    \item[] Answer: \answerYes{} 
    \item[] Justification: The paper provides information about the statistical significance of the experiments in Appendix~\ref{appe:stat}.

\item {\bf Experiments compute resources}
    \item[] Question: For each experiment, does the paper provide sufficient information on the computer resources (type of compute workers, memory, time of execution) needed to reproduce the experiments?
    \item[] Answer: \answerYes{} 
    \item[] Justification: The paper provides the compute resources used for the experiments in Appendix~\ref{appe:setting}.
    
\item {\bf Code of ethics}
    \item[] Question: Does the research conducted in the paper conform, in every respect, with the NeurIPS Code of Ethics \url{https://neurips.cc/public/EthicsGuidelines}?
    \item[] Answer: \answerYes{} 
    \item[] Justification: The research conducted in this paper conforms to the NeurIPS Code of Ethics in all aspects, including data collection, experimental design, and result dissemination, and involves no ethical violations or risks.

\item {\bf Broader impacts}
    \item[] Question: Does the paper discuss both potential positive societal impacts and negative societal impacts of the work performed?
    \item[] Answer: \answerYes{} 
    \item[] Justification: The paper discusses the potential societal impacts of this work in Appendix~\ref{appe:impact}.
    
\item {\bf Safeguards}
    \item[] Question: Does the paper describe safeguards that have been put in place for responsible release of data or models that have a high risk for misuse (e.g., pre-trained language models, image generators, or scraped datasets)?
    \item[] Answer: \answerNA{} 
    \item[] Justification: The paper does not involve the release of data or models with a high risk for misuse.

\item {\bf Licenses for existing assets}
    \item[] Question: Are the creators or original owners of assets (e.g., code, data, models), used in the paper, properly credited and are the license and terms of use explicitly mentioned and properly respected?
    \item[] Answer: \answerYes{} 
    \item[] Justification: The paper properly cites the existing datasets and open-source code used in this work and respects their licenses and terms of use.

\item {\bf New assets}
    \item[] Question: Are new assets introduced in the paper well documented and is the documentation provided alongside the assets?
    \item[] Answer: \answerYes{}     
    \item[] Justification: The paper provides anonymized code for training, and no new dataset is released.

\item {\bf Crowdsourcing and research with human subjects}
    \item[] Question: For crowdsourcing experiments and research with human subjects, does the paper include the full text of instructions given to participants and screenshots, if applicable, as well as details about compensation (if any)? 
    \item[] Answer: \answerNA{} 
    \item[] Justification: The paper does not involve crowdsourcing or research with human subjects.

\item {\bf Institutional review board (IRB) approvals or equivalent for research with human subjects}
    \item[] Question: Does the paper describe potential risks incurred by study participants, whether such risks were disclosed to the subjects, and whether Institutional Review Board (IRB) approvals (or an equivalent approval/review based on the requirements of your country or institution) were obtained?
    \item[] Answer: \answerNA{} 
    \item[] Justification: The paper does not involve crowdsourcing or research with human subjects.

\item {\bf Declaration of LLM usage}
    \item[] Question: Does the paper describe the usage of LLMs if it is an important, original, or non-standard component of the core methods in this research? Note that if the LLM is used only for writing, editing, or formatting purposes and does \emph{not} impact the core methodology, scientific rigor, or originality of the research, declaration is not required.
    \item[] Answer: \answerNA{}
    \item[] Justification: The core method development in this research does not involve LLMs as important, original, or non-standard components.

\end{enumerate}

\end{document}